\documentclass{article}

\usepackage{microtype}
\usepackage{graphicx}
\usepackage{subfigure}
\usepackage{booktabs} 
\usepackage{natbib}
\usepackage{fullpage}
\usepackage{hyperref}
\usepackage[usenames,dvipsnames]{xcolor}

\newcommand{\Sc}{\mathcal{S}}
\newcommand{\T}{\top}

\newcommand{\inflam}{\inf_{\substack{ \bm\lambda > 0 \\ \|\bm\lambda\|_2 =1 }}}

\usepackage{amsmath,amsfonts,bm}

\def\1{\bm{1}}

\def\vlambda{{\bm\lambda}}

\def\va{{\bm{a}}}
\def\vb{{\bm{b}}}
\def\vc{{\bm{c}}}

\def\vq{{\bm{q}}}

\def\vs{{\bm{s}}}

\def\vu{{\bm{u}}}
\def\vv{{\bm{v}}}
\def\vw{{\bm{w}}}
\def\vx{{\bm{x}}}

\def\mB{{\bm{B}}}

\def\mD{{\bm{D}}}

\def\mI{{\bm{I}}}

\def\mM{{\bm{M}}}

\def\mP{{\bm{P}}}
\def\mQ{{\bm{Q}}}

\def\mU{{\bm{U}}}
\def\mV{{\bm{V}}}
\def\mW{{\bm{W}}}
\def\mX{{\bm{X}}}

\DeclareMathAlphabet{\mathsfit}{\encodingdefault}{\sfdefault}{m}{sl}
\SetMathAlphabet{\mathsfit}{bold}{\encodingdefault}{\sfdefault}{bx}{n}

\newcommand{\R}{\mathbb{R}}

\DeclareMathOperator{\rank}{rank}
\DeclareMathOperator{\sspan}{span}
\DeclareMathOperator{\Tr}{Tr}

\usepackage{amsmath}
\usepackage{amssymb}
\usepackage{mathtools}
\usepackage{amsthm}

\usepackage[capitalize,noabbrev]{cleveref}
\crefname{section}{\S}{\S}
\crefformat{section}{\S#2#1#3}

\theoremstyle{plain}
\newtheorem{theorem}{Theorem}[section]
\newtheorem{proposition}[theorem]{Proposition}
\newtheorem{lemma}[theorem]{Lemma}
\newtheorem{corollary}[theorem]{Corollary}
\theoremstyle{definition}

\theoremstyle{remark}

\usepackage[textsize=tiny]{todonotes}

\title{The Role of Linear Layers in Nonlinear Interpolating Networks}
\author{Greg Ongie\footnote{G. Ongie is with the Department of Mathematical and Statistical Sciences, Marquette University, Milwaukee, WI, USA. e-mail: gregory.ongie@marquette.edu}~~\& Rebecca Willett\footnote{R. Willett is with the Department of Statistics and Department of Computer Science, University of Chicago, Chicago, IL, USA.}}

\begin{document}
\maketitle

\begin{abstract}
This paper explores the implicit bias of overparameterized neural networks of depth greater than two layers. Our framework considers a family of networks of varying depth that all have the same {\em capacity} but different implicitly defined {\em representation costs}. 
The representation cost of a function induced by a neural network architecture is the minimum sum of squared weights needed for the network to represent the function; it reflects the function space bias associated with the architecture.
Our results show that adding linear layers to a ReLU network
yields a representation cost that reflects a complex interplay between the alignment and sparsity of ReLU units.
Specifically, using a neural network to fit training data with minimum representation cost yields an interpolating function that is constant in directions perpendicular to a low-dimensional subspace 
on which a parsimonious interpolant exists.
\end{abstract}

\section{Introduction}
\label{sec:intro}

An outstanding problem in understanding the generalization properties of overparameterized neural networks is to characterize which functions are best represented by neural networks of varying architectures. Past work explored the notion of {\em representation costs} -- i.e., how much does it ``cost'' for a neural network to represent some function $f$. 
Specifically, the representation cost of a function $f$ is the minimum sum of squared network weights necessary for the network to represent $f$.

The following key question then arises:
\textbf{How does network depth affect which functions have minimum representation cost?} 
For instance, given a set of training samples, say we find the interpolating function that minimizes the representation cost; how is that interpolant different for a network with three layers instead of two layers? Both functions have the same values on the training samples, but they may have very different behaviors elsewhere in the domain. 

In this paper, we describe the representation cost of a family of networks with $L$ layers in which $L-1$ layers have linear activations and the final layer has a ReLU activation. As detailed in \cref{sec:related}, networks related to this class play an important role in both theoretical studies of neural network generalization properties and experimental efforts. {One reason that this is a particularly important family to study is that
adding linear layers does not change the capacity or expressively of a network, even though the number of parameters may change; this means that different behaviors for different depths  solely reflects the role of depth and not of capacity.}

We show that adding linear layers to a ReLU network with weight decay regularization is akin to using a two-layer ReLU network with 
nuclear or Schatten norm regularization on the weight matrix. This insight suggests that lower-rank weight matrices, corresponding to aligned ReLU units, will be favored.
However, the representation costs we derive provide a nuanced perspective that extends beyond ``linear layers promote alignment'', as illustrated in \cref{fig:first_example} and \cref{fig:counterintuitive}. In particular, the effect of linear layers, as understood through the corresponding representation costs, reflects a subtle interplay between ReLU unit alignment and the magnitudes of the outer layer weights. We find that lower representation costs are associated with  functions that can be parsimoniously expressed using only the orthogonal projection of their inputs onto a low-dimensional subspace.

\subsection{Related work}
\label{sec:related}
Past work has explored the role of neural network depth via a ``depth separation'' analysis (e.g. \cite{daniely2017depth,vardi2020neural}); these  analyses identify functions which may be efficiently represented at one depth but require an exponential width to represent them with fewer layers. This line of work has yielded important insights into the role of depth, but recent work has highlighted how functions leading to depth separation results are often highly oscillatory and  perhaps not fully capturing the import of depth in practical settings (i.e., may be ``worse case'' but not ``average case'' results).
In particular, \cite{safran2019depth}  shows that if the Lipschitz constant of the target function is kept fixed, then  existing depth separation results between 2- and 3-layer nets do not hold.

A number of papers have studied representation costs and implicit regularization from a  function space perspective associated with neural networks.
Following a univariate analysis by \cite{savarese2019infinite},  \cite{ongie2019function} considers two-layer multivariate ReLU networks where the hidden layer has
infinite width:
$$
\lim_{K \rightarrow \infty} \sum_{k=1}^K a_k[\vw_k^\top \vx + b_k]_+.
$$
Recent work by \cite{mulayoff2021implicit} connects the function space representation costs of two-layer ReLU networks to the stability of SGD minimizers.

\cite{gunasekar2018implicit} shows that $L$-layer {\em linear} networks with \textit{diagonal} structure induces a non-convex implicit bias over network weights corresponding to the $\ell^q$ norm of the outer layer weights for $q = 2/L$; similar conclusions hold for deep \textit{linear} convolutional networks. 
Recent work by \cite{dai2021representation} examines the representation costs of deep \textit{linear} networks  from a function space perspective. 
However, the existing literature does not 
fully characterize the representation costs of \textit{deep, non-linear} networks from a function space perspective.
\cite{parhi2021banach} consider deeper networks and define a compositional function space with a corresponding representor theorem; 
the properties
of this function space and the role of depth are an area of active investigation.

Our paper focuses on the role of linear layers in \textit{nonlinear} networks. The role of linear layers in such settings has been explored in a number of works.
\citet{golubeva2020wider} looks at the role of network \emph{width} when the    number of parameters is held fixed; it specifically looks at increasing    the width without increasing the number of parameters by adding    linear layers.
This procedure seems to help with generalization
performance (as long as the training error is controlled). However,
\citet{golubeva2020wider} 
note that
the implicit regularization caused by this
approach is not understood. 
\textit{One of the main contributions of our paper is a better understanding of this implicit regularization.}

\sloppypar
The effect of linear layers on training speed was previously examined by \cite{ba2013deep,urban2016deep}. \cite{arora2018optimization} considers implicit acceleration in deep nets and claims that depth induce a momentum-like term in training deep \emph{linear} networks with SGD, though the regularization effects of this acceleration are not well understood.  Implicit regularization of gradient descent has been studied in the context of matrix and tensor factorization problems \cite{gunasekar2018implicit,arora2019implicit,razin2020implicit,razin2021implicit}. Similar to this work, low-rank representations play a key role in their analysis. 

\begin{figure}[ht!]
    \centering
    \subfigure[$L=2$ layers]{
    \includegraphics[width=0.3\columnwidth]{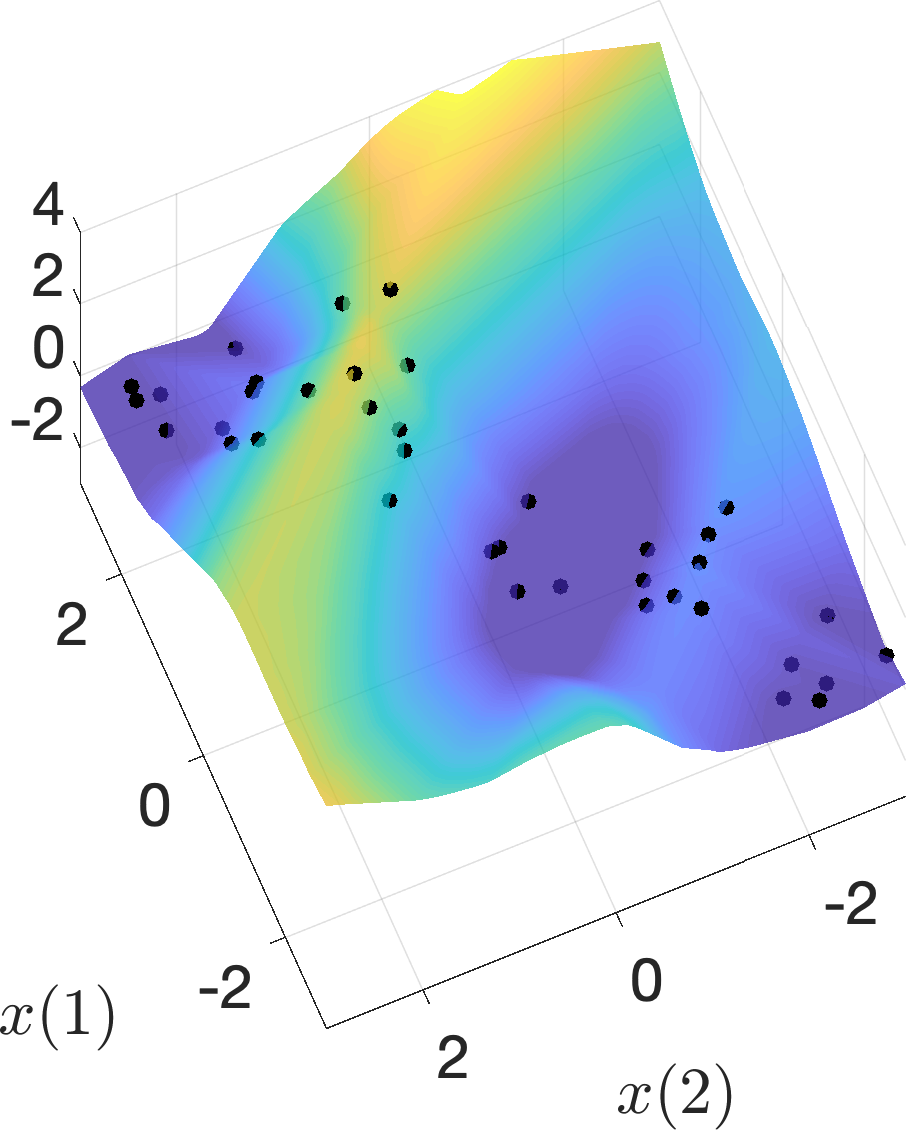}}
    \subfigure[$L=3$ layers]{
    \includegraphics[width=0.3\columnwidth]{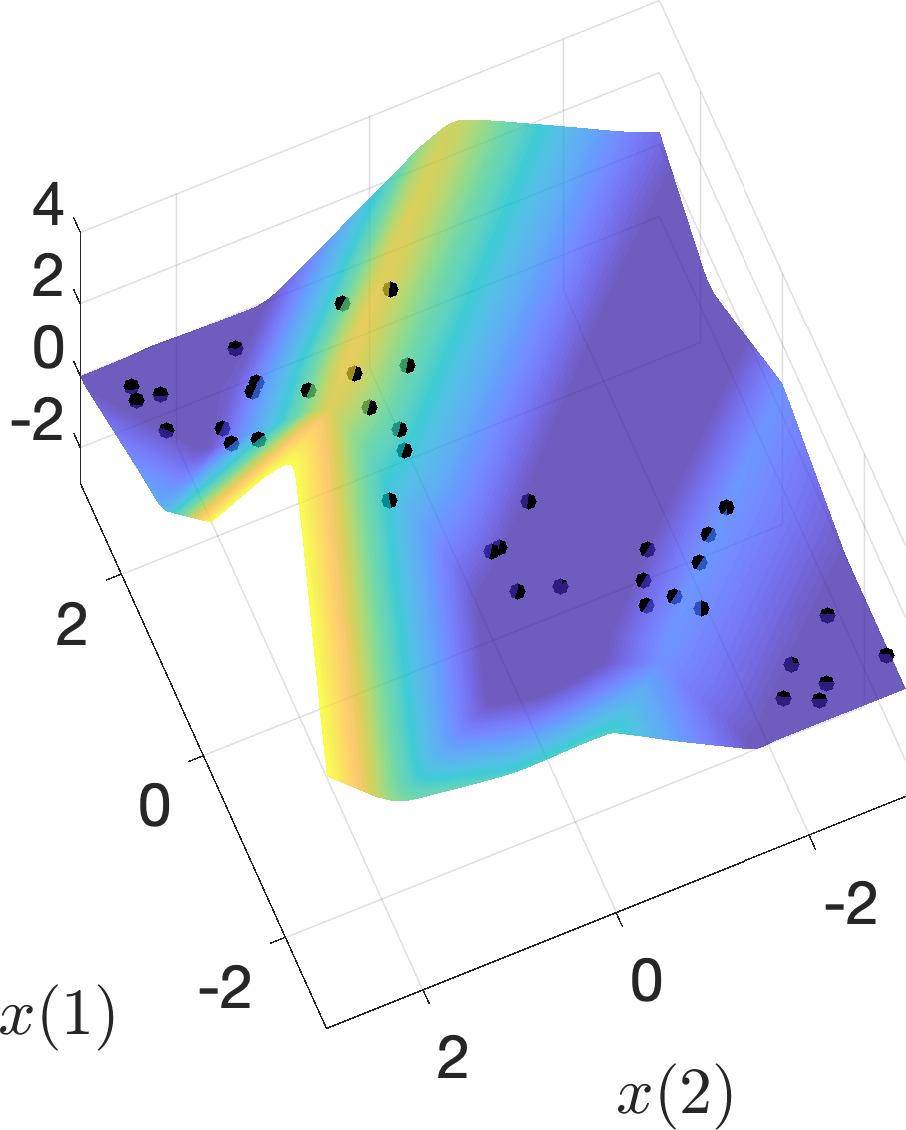}
    }
    \subfigure[$L=4$ layers]{
    \includegraphics[width=0.3\columnwidth]{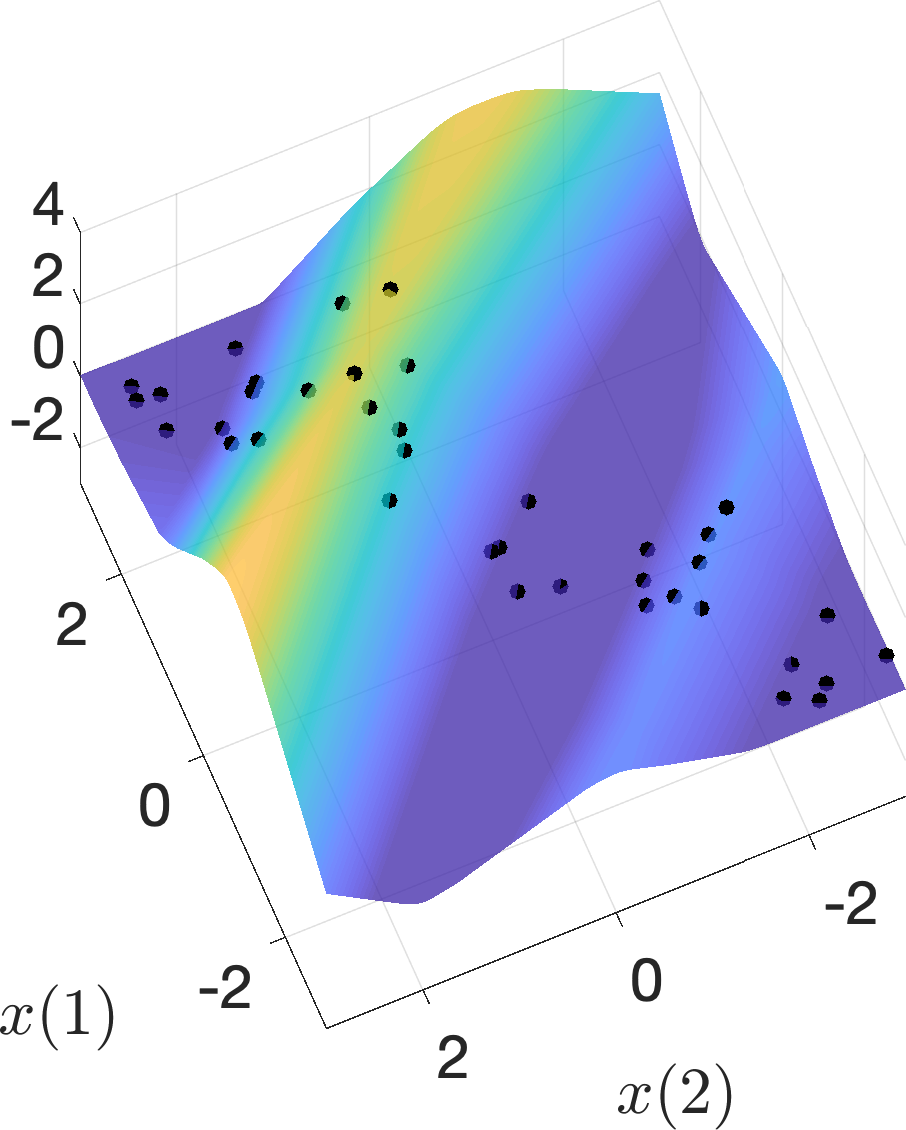}
    }
    \caption{\textbf{Numerical evidence that weight decay promotes unit alignment with more linear layers.} Neural networks with $L-1$ linear layers plus one ReLU layer were trained using SGD with weight decay regularization to close to zero training loss on the training samples, as shown in black. Pictured in (a)-(c) are the resulting interpolating functions shown as surface plots. 
    Our theory predicts that as the number of linear layers increases, the learned interpolating function will become closer to constant in directions perpendicular to a low-dimensional subspace on which a parsimonious interpolant can be defined.}
    \label{fig:first_example}
\end{figure}

\subsection{Notation}
For a vector $\va\in \R^K$, we use $\|\va\|_p$ to denote its $\ell^p$ norm. For a matrix $\mW$, we use $\|\mW\|_F$ to denote the Frobenius norm, $\|\mW\|_*$ to denote its nuclear norm (i.e., the sum of the singular values), and for $0 < q \leq 1$ we use
$\|\mW\|_{\Sc^q}$ to denote its Schatten-$q$ quasi-norm (i.e., the $\ell^q$ quasi-norm of the singular values of a matrix $\mW$). Given a vector $\va \in \R^{K}$, the matrix $\mD_{\va} \in \R^{K \times K}$ is a diagonal matrix with the entries of $\va$ along the diagonal. For a vector $\bm\lambda$, we write $\bm\lambda > 0$ to indicate it has all positive entries. Finally, we use $[t]_+ = \max\{0,t\}$ to denote the ReLU activation, and whose application to vectors is understood entrywise.

\section{Definitions}
Let $N_2(\R^d)$ denote the space of functions expressible as a two-layer ReLU network having input dimension $d$ and such that the width $K$ of the single hidden layer is unbounded. Every function in $N_2(\R^d)$ is described (non-uniquely) by a collection of weights $\theta = (\mW,\va,\vb,c)$:
\begin{align}
h_\theta^{(2)}(\vx) & = \va^\T[\mW\vx + \vb]_+ + c.\\
& = \sum_{k=1}^K a_k[\vw_k^\T\vx + b_k]_+ + c
\end{align}
with $\mW \in \R^{K\times d}$, $\va,\vb \in \R^K$ and $c\in \R$. We denote the set of all such parameter vectors $\theta$ by $\Theta_2$.

In this work, we consider a re-parameterization of networks in $N_2(\R^d)$. Specifically, we replace the linear input layer $\mW$  with $L-1$ linear layers:
\begin{align}
    h_\theta^{(L)}(\vx) & = \va^\T[\mW_{L-1}\cdots\mW_2\mW_1 \vx + \vb]_+ + c
\end{align}
where now $\theta = (\mW_1,\mW_2,...,\mW_{L-1},\va,\vb,c)$. Again, we allow the widths of all layers to be arbitrarily large. Let $\Theta_L$ denote the set of all such parameter vectors. With any $\theta \in \Theta_L$  we associate the cost
\begin{equation}
    C_L(\theta) = \frac{1}{L}\left(\|\va\|_2^2 + \|\mW_{1}\|_F^2 + \cdots +  \|\mW_{L-1}\|_F^2\right),
\end{equation}
i.e., the squared Euclidean norm of all non-bias weights.

Given training pairs $\{(\vx_i,y_i)\}_{i=1}^n$, consider the problem of finding a $L$-layer network with minimal cost $C_L$ that interpolates the training data:
\begin{equation}\label{eq:opt1}
\min_{\theta \in \Theta_L} C_L(\theta)~~s.t.~~h^{(L)}_\theta(\vx_i) = y_i
\end{equation}
This optimization is akin to training a network to interpolate training data using SGD with squared $\ell^2$ norm or weight decay regularization \cite{hanson1988comparing,loshchilov2017decoupled}. 
We may recast this as an optimization problem in function space: for any $f \in N_2(\R^d)$, define its $L$-layer representation cost $R_L(f)$ by
\begin{equation}\label{eq:RLdef}
    R_L(f) = \min_\theta C_L(\theta)~~s.t.~~f = h^{(L)}_\theta.
\end{equation}
Then \eqref{eq:opt1} is equivalent to:
\begin{equation}\label{eq:opt2}
\min_{f \in N_2} R_L(f)~~s.t.~~f(\vx_i) = y_i.
\end{equation}
Earlier work such as \cite{savarese2019infinite} has shown that
\begin{align}\label{eq:R2}
R_2(f) =& \min_{\theta \in \Theta_2} \|\va\|_1~~\text{s.t.}~~\|\vw_k\|_2 =1,\\
&\qquad \forall k=1,...,K~\text{and}~f = h^{(2)}_\theta \nonumber
\end{align}
\textit{Our goal is to characterize the representation cost $R_L$ for different numbers of layers $L \geq 3$, and describe how the set of global minimizers of \eqref{eq:opt2} changes with $L$, providing insight into the role of linear layers in nonlinear ReLU networks.}

\begin{figure}[ht!]
    \centering
    \subfigure[\textit{Left}: Minimum $R_2$ interpolant when samples lie on two rays separated by an angle of $0.55\pi$. \textit{Right}: Minimum $R_3$ interpolant of same data.]{\includegraphics[width=.7\linewidth]{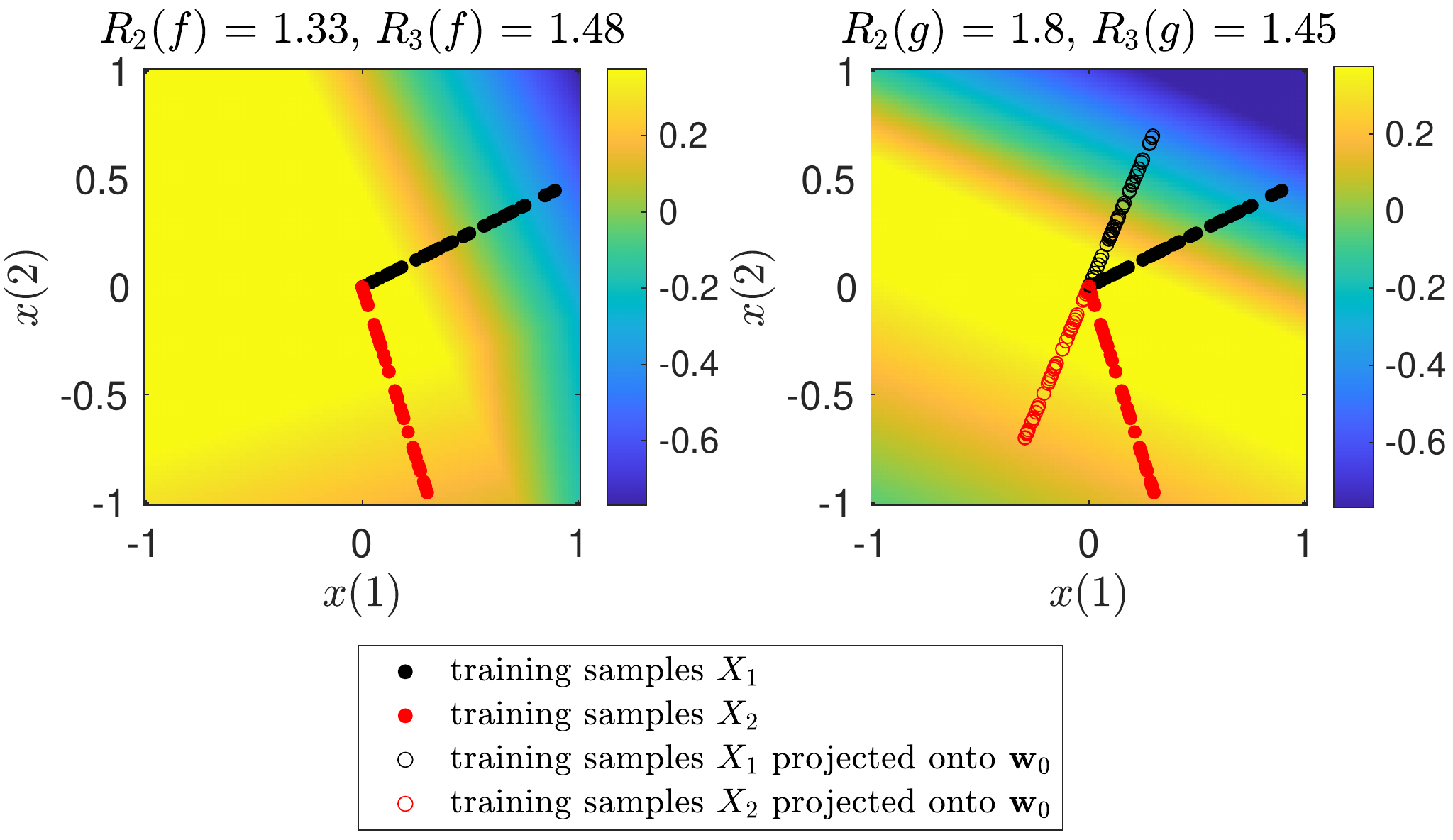}}
    \subfigure[\textit{Left}: Minimum $R_2$ interpolant when samples lie on two rays separated by an angle of $0.52\pi$. \textit{Right}: Interpolant of same data with aligned ReLU units \textbf{does not} minimize $R_3$.]{\includegraphics[width=.7\linewidth]{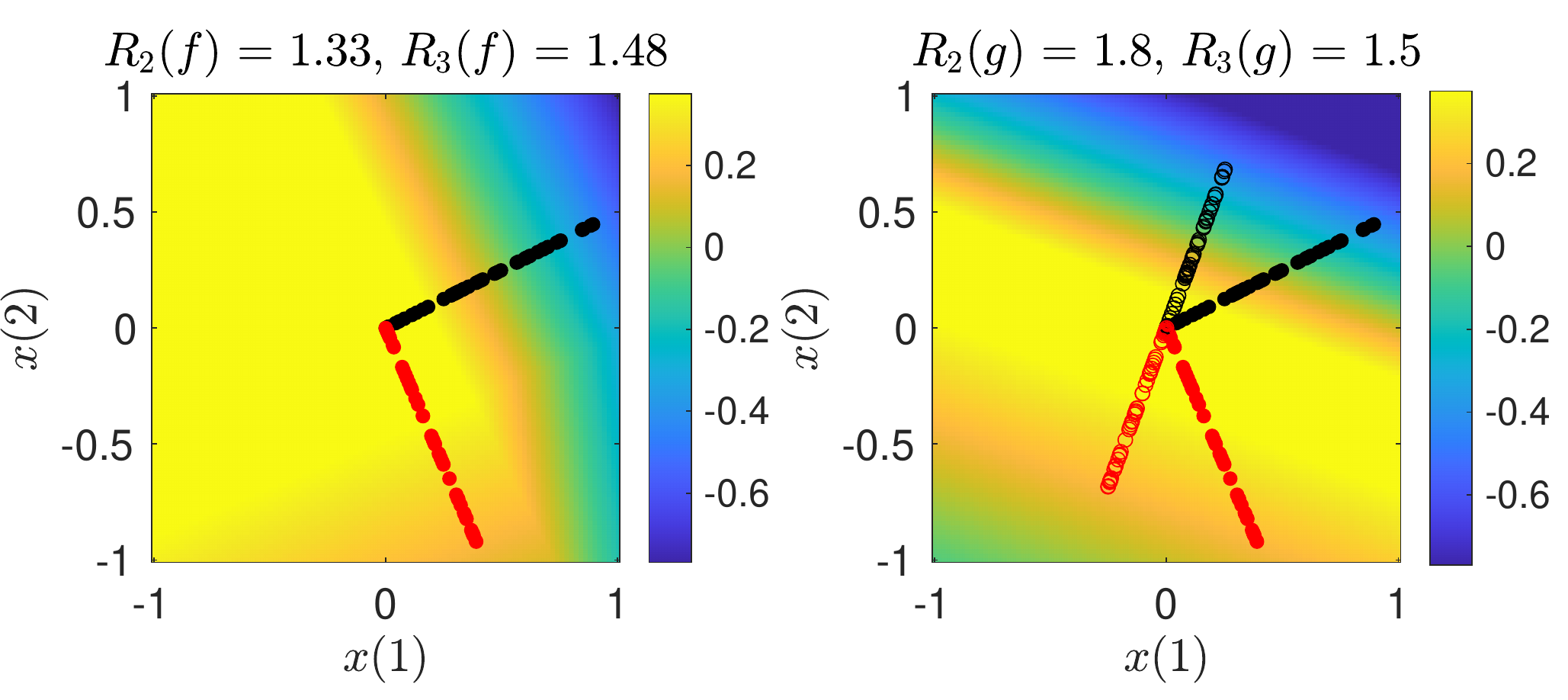}}
    \caption{\textbf{Adding linear layers does not always promote alignment of ReLU units.} (a) Samples lie on two rays separated by $0.55\pi$. The minimum $R_2$  interpolant $f$  is quite different from the minimum $R_3$ interpolant $g$, where the latter has ReLU units aligned with a subspace that is the difference between the two rays. (b) Here the rays are only separated by $0.52\pi$, and the minimum $R_2$ interpolant $f$ also has a smaller $R_3$ representation cost than the interpolant with aligned ReLU units. This example illustrates that \textit{ReLU alignment alone does not capture representation costs of deeper networks.}
    }
    \label{fig:counterintuitive}
\end{figure}

\section{Simplifying the Representation Cost}\label{sec:simplify}
Here we derive simplified expressions for the representation costs $R_L$ with $L\geq 3$. Proofs of all results in this section are given in \cref{app:repcost}

Our first result shows that if the predictor function is univariate then the $R_L$ representation cost reduces to the $2/L$-power of the $R_2$ representation cost:
\begin{theorem}\label{thm:univar}
If $f \in N_2(\R^1)$ (i.e., $f$ is univariate) then
\begin{equation}
       R_L(f) = [R_2(f)]^{2/L}.
\end{equation}
\end{theorem}
This shows that $L$-layer minimum norm interpolants in 1-D coincide  with $2$-layer minimum norm interpolants, as characterized by \cite{savarese2019infinite,hanin2019universal}.

However, in the multivariate setting, where the input dimension $d>1$, the $R_L$-costs with $L\geq 3$ are not simply a monotonic transform of the $R_2$-cost, as we now show.

First, we prove that the general $R_L$-cost can be re-cast as an optimization over two-layer networks, but where the representation cost associated with the inner-layer weight matrix $\mW$ changes with $L$:

\begin{lemma}\label{lem:schatten} 
Suppose $f\in N_2(\R^d)$. Then 
\begin{equation}\label{eq:nucnormex}
R_L(f) = \min_{\theta \in \Theta_2} \tfrac{1}{L}\|\va\|_2^2 + \tfrac{L-1}{L}\|\mW\|^{q}_{\Sc^{q}}~~s.t.~~f = h_\theta^{(2)}
\end{equation}
where  $q:=2/(L-1)$ and $\|\mW\|_{\Sc^q}$ is the Schatten-$q$ quasi-norm, i.e., the $\ell^q$ quasi-norm of the singular values of $\mW$.
\end{lemma}
Note that Schatten-$q$ quasi-norms with $0 < q \leq 1$ are often used as a surrogate for the rank penalty. Intuitively, this shows that minimizing the $R_L$-cost for $L \geq 3$ ought to promote low-rank inner-layer weight matrices $\mW$, and this bias should become more pronounced as $L$ grows. However, the reduced form of the $R_2$-cost in \eqref{eq:R2} suggests that sparsity of the outer-layer weights $\va$ ought to also play a role in determining the $R_L$-cost for $L>2$, yet this dependence is not explicitly revealed in \eqref{eq:nucnormex}.

Part of the difficulty in interpreting the expression for the $R_L$-cost in \eqref{eq:nucnormex} is that it varies under different sets of parameters realizing the same function. In particular, the loss in $\eqref{eq:nucnormex}$ may vary under a trivial rescaling of the weights:  for any vector $\bm\lambda \in \R^K$ with positive entries, by the 1-homogeneity of the ReLU activation  we have  $$\va [\mW\vx + \vb]_+ + c= \va\mD_{\bm\lambda}^{-1}[\mD_{\bm\lambda}\mW\vx + \mD_{\bm\lambda}\vb]_+ + c$$ 
However, the value of the objective in \eqref{eq:nucnormex} may vary between the two parameter sets $\theta = (\mW,\va,\vb,c)$ and $\theta' = (\mD_{\bm\lambda}\mW,\va_{\bm\lambda}\mD_{\bm\lambda}^{-1},\mD_{\bm\lambda}\vb,c)$ realizing the same function.

To account for this scaling invariance, we define a new loss function $\Phi_L$ on pairs of inner- and outer-layer weights $(\mW,\va)$ by optimizing over all such ``diagonal'' rescaling of units:
\begin{align}
\Phi_L(\mW,\va) :=  \inf_{\substack{\bm\lambda \in \R^K\\ \lambda_k > 0,\forall k}} \tfrac{1}{L}\|\va\mD_{\bm\lambda}^{-1}\|_2^2 + \tfrac{L-1}{L}\|\mD_{\bm\lambda}\mW\|_{\Sc^{q}}^{q}
\end{align}
where  $q:=2/(L-1)$.

Since the diagonal rescaling of units does not change the function represented by the network, we may replace the objective in \eqref{eq:nucnormex} with $\Phi_L(\mW,\va)$, which gives us the following equivalent expression for the $R_L$-cost:
\begin{lemma}
For any $f\in N_2(\R^d)$ we have
\begin{equation}
    R_L(f) = \min_{\theta \in \Theta_2} \Phi_L(\mW,\va) ~~s.t.~~f = h^{(2)}_\theta.
\end{equation}
\end{lemma}

Previous work \cite{neyshabur2015norm,neyshabur2017exploring, savarese2019infinite} has shown that in the case of $L=2$ (i.e., a single hidden-layer ReLU network with no additional linear layers), we have
\begin{align}
    \Phi_2(\mW,\va) = \sum_{k=1}^K |a_k|\|\vw_k\|_2.
\end{align}
This has been referred to as the ``path norm'' by \cite{neyshabur2017exploring}. Further constraining $\|\vw_k\|_2=1 \;\forall \; k$, then $\Phi_2(\mW,\va) = \|\va\|_1$, which gives the simplified $R_2$-cost in \eqref{eq:R2}. 

Our results suggest that no such closed-form formula exists for $\Phi_L$ with $L\geq 3$. However, the following lemma $\Phi_L$ for $L\geq 3$ gives a useful further reduction of $\Phi_L$, which is central to the results in \cref{sec:interp} and \cref{sec:example}.
\begin{lemma}\label{lem:phiLinfdef} For any $\mW \in \R^{K\times d}$ and $\va \in \R^K$ we have
\begin{equation}
    \Phi_L(\mW,\va) = \inf_{\substack{\|\bm\lambda\|_2= 1 \\ \lambda_k > 0,\forall k}} \|\mD_{\bm\lambda}^{-1}\mD_{\va}\mW\|_{\Sc^q}^{2/L}
    \label{eq:phi3}
\end{equation}
where $q = 2/(L-1)$.
\end{lemma}

Below, we describe some further simplifications of $\Phi_L$ for special configurations of inner-layer weight matrices, and general upper and lower bounds.

The following result shows that for a certain class of $\mW$ matrix, $\Phi_L(\mW,\va)$ reduces to a group sparsity penalty on the vector of outer-layer weights, where groups correspond to clusters of co-linear rows of $\mW$ such that vectors associated with each cluster are mutually orthogonal.

\begin{proposition}\label{prop:orthorows} Suppose each row of $\mW \in \R^{K\times d}$ belongs to a set $\{\pm\vv_1,...,\pm\vv_m\}$ such that $\vv_1,...,\vv_m$ are orthonormal. For all $j=1,..,m$, let $\va_j$ be the vector containing the subset of outer-layer weights corresponding to rows of $\mW$ equal to $\pm\vv_j$. Then we have
\begin{equation}
    \Phi_L(\va,\mW) = \sum_{j=1}^m \|\va_j\|_1^{2/L}.
\end{equation}
\end{proposition}

Two extremes of the above proposition are illustrated by the following corollaries:

\begin{corollary}\label{cor:rankone}
Suppose $\mW \in \R^{K\times d}$ is rank-one and has unit-norm rows and  $\va\in\R^K$ is arbitrary. Then
\begin{align}
     \Phi_L(\mW,\va) & = \|\va\|_1^{2/L}.
\end{align}

\end{corollary}

\begin{corollary}\label{cor:orthounits} Suppose the rows of $\mW \in \R^{k\times d}$ are orthonormal and  $\va\in\R^k$ is arbitrary. Then
\begin{align}
     \Phi_L(\mW,\va) =& \|\va\|_{2/L}^{2/L}.
\end{align}
\end{corollary}

Finally, we give some results that are particular to the $L = 3$ layer case. Note that by  \cref{lem:phiLinfdef}, the $\Phi_3$ loss involves minimizing over the nuclear norm of a matrix, which is a convex penalty. This allows us to give the following alternative characterization of $\Phi_3$ by way of convex duality:
\begin{lemma} 
\label{lem:Q}
For any $\mW = [\vw_1~ \vw_2~ ... ~\vw_K]^\T \in \R^{K\times d}$ and $\va \in \R^K$ we have
\begin{equation}\label{eq:Qeq}
\Phi_3(\mW,\va) = \max_{\|\mQ\|_2 \leq 1} \sum_{k=1}^K\left| a_k \langle \vq_k, \vw_k\rangle\right|^{2/3}
\end{equation}
where the dual variable $\mQ = [\vq_1~ \vq_2~ ...~ \vq_K]^\T \in \mathbb{R}^{K\times d}$ has the same dimensions as $\mW$ and $\|\mQ\|_2$ denotes the spectral norm of $\mQ$ (i.e., the maximum singular value of $\mQ$). 
\end{lemma}
The benefit of \cref{lem:Q} is that it allows us to easily generate lower bounds for $\Phi_3(\mW,\va)$, simply by evaluating the objective in \eqref{eq:Qeq} at any matrix $\mQ$ with $\|\mQ\|_2\leq 1$.

Next, we give an upper-bound for $\Phi_3$ that quantifies the interplay between low-rankness of the inner-layer weight matrix and the sparsity of the outer-layer weights:
\begin{theorem} \label{thm:UB}
Suppose $\mW \in \R^{K\times d}$ 
is a rank-$r$ matrix, and let $\mW = \mU\bm\Sigma\mV^\T$ be a (thin) SVD, such that $\mU \in \R^{K\times r}$, $\bm\Sigma = \text{diag}(\sigma_1,...,\sigma_r) \in \R^{r\times r}$, and $\mV \in \R^{d\times r}$. Let $\va \in \R^K$ be arbitrary. Then
\begin{subequations}
\begin{align}
\Phi_3(\mW,\va) & \leq \sum_{j=1}^r \left( \sigma_j \sum_{k=1}^K |a_k u_{k,j}|\right)^{2/3} \label{eq:phi3ub}
\\
& = \sum_{j=1}^r \left(\sum_{k=1}^K |a_k \langle \vw_k,\vv_j\rangle|\right)^{2/3}
\end{align}
\end{subequations}
Furthermore, equality holds when $\mW$ satisfies the conditions of \cref{prop:orthorows}.
\end{theorem}

Recall the definition of the “entry-wise” $\ell_{p,q}$ norm of a matrix $\mB \in \mathbb{R}^{m \times n}$:
$$
\|\mB\|_{p,q} = \left(  \sum_{j=1}^n \left( \sum_{i = 1}^m  |b_{i,j}|^p\right)^{q/p}\right)^{1/q};
$$
a classic example is $\|\mB\|_{2,1}$ used for group-sparse regularization in sparse coding \cite{yuan2006model}. Note that \cref{prop:orthorows} is equivalent to  
$$
\Phi(\mW,\va) \le \|\mB\|_{1,\frac{2}{3}}^{\frac{2}{3}}  \text{ where }  \mB:= \mD_{\va} \mW \mV = \mD_{\va} \mU \bm\Sigma
$$
and $\mD_{\va} \in \mathbb{R}^{K \times K}$ is a diagonal matrix with the entries of $\va$ along the diagonal.
This framing highlights that the representation cost depends not only on the alignment of the ReLU units with one another (as represented by the diagonal elements of $\mathbf{\Sigma}$ and entries of $\mU$), but also the sparsity of the weights on them (i.e. the sparsity of $\va$).

\section{Minimal $R_L$ Interpolating Solutions}\label{sec:interp}

Minimum $R_L$-cost interpolants of a finite set of data can reveal important features of representation costs and their impacts. In overparameterized neural networks, there are typically many possible interpolants, and representation costs guide which of those interpolants would be selected when we fit the data using weight decay. In this section, we consider two key settings: (a) when the training features are supported on a subspace, and (b) when the training features are not supported on a subspace, but an interpolant exists which is a function of the \textit{projection} of the features onto a subspace. 
The latter case is typical of overparameterized settings.
All proofs of results in this section are given in  \cref{app:mincost}.

\subsection{Training features contained in a subspace}\label{subsec:onesub}
We prove that in the special case where the training features are entirely contained in a subspace, every minimum $R_L$-cost interpolating  solution must depend on only the projection of features onto that subspace:
\begin{proposition}\label{prop:colineardata}
Let $\Sc \subset \R^d$ denote the subspace spanned by the training features $\{\vx_i\}_{i=1}^n$. Given any set of training labels $\{y_i\}_{i=1}^n$, let $f$ be any minimum $R_L$-cost interpolating solution for any $L\geq 2$. Then 
\[
f(\vx) = f(\mP_{\Sc}\vx)
\]
for all $\vx\in\R^d$, where $\mP_{\Sc}$ is the orthogonal projector onto $\Sc$.
\end{proposition}

More generally, since the representation cost is translation invariant, the above result extends to the case where the training features span an affine subspace $\mathcal{A} := \{\vv + \vx : \vx \in \mathcal{S}\}$ where $\vv \in \Sc^\perp$, in which case we have $f(\vx) = f(\mP_{\Sc}\vx + \vv)$.

The above proposition implies that any minimum $R_L$-cost interpolant $f$ will have all its units aligned with ${\cal S}$; i.e., every inner-layer weight vector $\vw_k \in {\cal S} \;\forall\; k$. Thus, $f$ will be constant in directions orthogonal to ${\cal S}$. 

Specializing this result to the case of training features constrained to a one-dimensional subspace, we see that all minimum $R_L$-cost interpolants must have units aligned along the subspace (rank-one inner-layer weight matrix). Combined with \cref{cor:rankone}, this gives the immediate corollary:
\begin{corollary}\label{prop:colineardata2}
If the training features $\vx_i$ are co-linear (i.e., there exist vectors $\vu,\vv \in \R^d$ such that $\vx_i = t_i\vu + \vv$ for some scalars $t_i$), then given any set of training labels $\{y_i\}_{i=1}^n$, the collection of minimum $R_L$-cost interpolating solutions is identical for all $L\geq 2$. Furthermore, every such minimizer $f$ has aligned units, meaning it can be written in the form $f(\vx) = \sum_{k=1}^K a_k [s_k\vu^\T\vx +b_k]_+ + c$, where $s_k =\pm 1$.
\end{corollary}

The above results do not depend on the number of linear layers. However, next we show there are settings where minimum $R_L$-cost interpolating solutions differ for $L=2$ and $L=3$.

\subsection{Representations supported on a subspace}
\label{sec:subspaces}
Suppose that training samples may be interpolated by a function
\begin{align}
    f_*(\vx) =& \sum_{k=1}^K a_k[\vw_k^\top \vx + b_k]_+ + c,
\end{align}
and assume that $\|\vw_k\| = 1 \; \forall \; k$. Given a subspace ${\cal S}$ and corresponding orthogonal projection operator $\mP_{\cal S}$ where $\mP_{\cal S} \vw_k \neq 0 \; \forall \; k$, we construct the function 
\begin{align}
    g_{\cal S}(\vx) =& \sum_{k=1}^K \tilde a_k [\tilde \vw_k ^\top \vx + \tilde b_k]_+ + c
\end{align}
where
\begin{align}
    \tilde \vw_k := \frac{\mP_{\cal S} \vw_k}{\|\mP_{\cal S} \vw_k\|_2},\qquad \tilde a_k := \frac{a_k}{r_k},\qquad
\tilde b_k := \frac{b_k a_k}{\tilde a_k},
\end{align}
and $r_k$ is defined as follows. Let $\mX_k \in \R^{d \times n_k}$ be a matrix of the $n_k \le n$ training samples that are ``active'' under ReLU unit $k$ -- that is,  $\vx_i$ is a column of $\mX_k$ if $\vw_k^\top \vx_i + b_k > 0$. Further define $\Sigma_k := \mX_k \mX_k^\top$. Then we define
\begin{align}
    r_k := \frac{\vw_k^\top \Sigma_k\mP_{\cal S} \vw_k}{\vw_k^\top \Sigma_k \vw_k \|\mP_{\cal S} \vw_k\|_2}.
\end{align}
If $r_k \neq 0 \;\forall \; k$, then, by construction, $g_{\cal S}$ interpolates the training samples since,
$$
a_k[\vw_k^\top \vx_i + b_k]_+ = \tilde a_k[\tilde \vw_k^\top \vx_i + \tilde b_k]_+ \;\forall\;i,k.
$$
Note that all $\tilde \vw_k \in {\cal S}$, but, unlike the setting in \cref{prop:colineardata}, we do \textit{not} assume that the training samples all lie in the subspace. 

Given this construction, we have
\begin{align}
    R_2(f_*) =& \sum_{k=1}^K |a_k| = \sum_{k=1}^K  |\tilde a_k r_k| \\
    R_2(g_{\cal S}) =& \sum_{k=1}^K  \left|\frac{a_k}{r_k}\right| = \sum_{k=1}^K  |\tilde a_k|.
\end{align}
We may then conclude that 
\[
R_2(f_*) \le R_2(g_{\cal S})
\]
whenever
\begin{align}
\left|
\frac{\vw_k^\top \Sigma_k\mP_{\cal S} \vw_k}{\vw_k^\top \Sigma_k \vw_k }
\right| \le \|P_{{\cal S}} \vw_k\|_2 \;\forall \; k. \label{eq:R2_r}
\end{align}
In other words,
the samples in $\mX_k$ must be more closely aligned with $\vw_k$ than $\mP_{\cal S}\vw_k$, and the ratio of these alignments must be bounded by how much of $\vw_k$'s energy is in ${\cal S}$. In this case, even though $g_{\cal S}$ is an interpolating function, it may not correspond to the minimum $R_2$ interpolant when the feature vectors are not in the subspace $\cal S$ (in contrast to the setting in \cref{subsec:onesub} where, when the samples all lie in ${\cal S}$, the minimum $R_2$ interpolant will have all inner-layer weight vectors in ${\cal S}$). 

As a special case, imagine $\mX_k = \vw_k \vc_k^\top$ -- that is, all training samples that activate ReLU unit $k$ lie along the subspace spanned by $\vw_k$, making $\rank(\mX_k)=1$, and let $\vw$ be an orthonormal basis for a one-dimensional ${\cal S}$. (This special case is examined in detail in \cref{sec:rays}.) In this case, the condition in \eqref{eq:R2_r} is always satisfied.

To understand the three-layer representation cost, 
we use \cref{lem:phiLinfdef} and compare 
\begin{align} 
\Phi_3(\mW,\va) = \inf_{\substack{\|\bm\lambda\|_2= 1 \\ \lambda_k > 0,\forall k}} \|\mD_{\vlambda}^{-1}\mD_{\va}\mW\|_*^{2/3}
\label{eq:phi3orig}
\end{align}
with
\begin{align}
\Phi_3(\tilde \mW,\tilde \va) = \inf_{\substack{\|\bm\lambda\|_2= 1 \\ \lambda_k > 0,\forall k}} \|\mD_{\vlambda}^{-1}\mD_{\tilde \va}\tilde \mW\|_*^{2/3}
\end{align}
corresponding to $R_3(f_*)$ and $R_3(g_{\cal S})$, respectively.
Define
\begin{align}
    q_k :=& r_k \|\mP_{\cal S} \vw_k\| = \frac{\vw_k^\top \Sigma_k\mP_{\cal S} \vw_k}{\vw_k^\top \Sigma_k \vw_k}
\end{align}
and $    \vq := (
    q_1, \ldots, q_K
   )^\top.$
Then 
\begin{align}
  \Phi_3(\tilde \mW,\tilde \va) = \inf_{\substack{\|\bm\lambda\|_2= 1 \\ \lambda_k > 0,\forall k}}  \| \mD_{\vq}^{-1} ( \mD_{\vlambda}^{-1} \mD_{\va} \mW) \mP_{\cal S}
\|_*^{2/3}.  \label{eq:phi3tilde}
\end{align}

Differences in the $R_3$ representation costs associated with a two-layer ReLU network 
\eqref{eq:phi3orig} and a two-layer ReLU network with an additional linear input layer \eqref{eq:phi3tilde} highlight the importance of \textit{both} the alignment of the ReLU units (via $\mP_{\cal S}$) and  their scales (via $\mD_{\vq}^{-1}$).  Specifically, the $\mP_{\cal S}$ factor ensures the product has $\rank = \dim(\cal S)$, and so the nuclear norm in 
\eqref{eq:phi3tilde} will often be smaller than that in \eqref{eq:phi3orig}. This is consistent with the intuition that an interpolating network with all ReLU units aligned with a low-dimensional subspace should have a smaller $R_3$ representation cost. However, despite this intuition, we show 
this is not always the case, and the vector $\vq$, which captures the alignment of the training data with the subspace $\cal S$ \textit{vis-\`{a}-vis} the interpolating function $f_*$, can sometimes result in 
\eqref{eq:phi3orig} being smaller than   \eqref{eq:phi3tilde}.

The expression in \eqref{eq:phi3tilde} does not admit a general analytic simplification. However, it may be computed exactly in some special cases, computed numerically, upper bounded using \cref{thm:UB}, and lower bounded using \cref{lem:Q} with any $Q$. In \cref{sec:example}, we highlight a special case in which \eqref{eq:phi3tilde} admits a simple analytical expression, allowing us to characterize the conditions under which the representation cost for the three-layer network is smaller when the ReLU units are aligned.

\section{Examples illustrating ReLU alignment}
\label{sec:example}

\subsection{Networks with ReLU weights of similar magnitudes}
Suppose we have two networks that interpolate the training data with the same $R_2$-cost, but one network has all its units aligned, and the other does not. Then the following result shows that the network with aligned units always has strictly lower $R_3$-cost.

\begin{proposition}\label{prop:rankoneR3}
Suppose $f$ and $g$ are such that $R_2(f) = R_2(g)$, i.e., $f$ and $g$ can be described by inner-layer and outer-layer weight pairs $(\mW_1,\va_2)$ and $(\mW_2,\va_2)$, respectively, where both $\mW_1$ and $\mW_2$ have unit-norm rows, and $\|\va_1\|_1 = \|\va_2\|_1$. If $\mW_1$ has rank greater than one, while $\mW_2$ is rank one, then $R_3(g) < R_3(f)$.
\end{proposition}
See \cref{app:proof51} for the proof.

This shows that if several networks interpolate the training data with the same $R_2$-cost, yet there is one having aligned units (i.e., rank-one inner-layer weight matrix $\mW$), the latter network is always the preferred fit according to the $R_3$-cost.

\subsection{Features on two rays}
\label{sec:rays}

Suppose the training features $X = \{\vx_1,...,\vx_n\} \subset \R^d$ can be partitioned into two sets $X_1$, $X_2$, such that $X_1 \subset R_1$ and $X_2\subset R_2$ where $R_1$ and $R_2$ are non-colinear rays separated by half-spaces, i.e., there exist unit vectors $\vw_1,\vw_2 \in \R^2$ with $\vw_1^\T\vw_2 < 0$ and $\vw_1\neq -\vw_2$,  such that every feature $\vx_i \in X_1$ has the form $\vx_i = c_i \vw_1$ for some $c_i>0$, and every feature in $\vx_j \in X_2$ has the form $\vx_j = d_j \vw_2$ for some $d_j > 0$. 

Let $f$ be any network interpolating the training data such that all units active over points in $X_1$ are aligned with $\vw_1$ and all units active over points in  $X_2$ are aligned with $\vw_2$:
\[
f(\vx) = \sum_{k=1}^{K_1}a_{1,k}[\vw_1^\T\vx + b_{1,k}]_+ + \sum_{k=1}^{K_2 }a_{2,k}[\vw_2^\T\vx + b_{2,k}]_+.
\]
Consider the related network $g$ defined by
\[
g(\vx) = \sum_{k=1}^{K_1}\tilde{a}_{1,k}[\vw_0^\T\vx + \tilde{b}_{1,k}]_+ +\sum_{k=1}^{K_2} \tilde{a}_{2,k}[-\vw_0^\T\vx  +\tilde{b}_{2,k}]_+,
\]
where $\vw_0 = \frac{\vw_1-\vw_2}{\|\vw_1-\vw_2\|}$ and $\tilde{a}_{j,k} := a_{j,k}/|\vw_j^\T\vw_0|$ and $\tilde{b}_{j,k} := |\vw_j^\T\vw_0|b_{j,k}$. This is the interpolating network obtained by $f$ replacing all inner-layer weight vectors with $\pm \vw_0$ and rescaling outer-layer weights and bias terms to satisfy interpolation constraints.

We prove that the network $f$ whose weights aligned along the two rays always has lower $R_2$-cost than associated network $g$ with all weights aligned in one direction. However, we also prove the reverse is true of their $3$-layer representation costs assuming the angle between the rays is not too large, and the size of the $\ell^1$-norms of the outer-layer weights of $f$ are sufficiently balanced:

\begin{proposition} 
\label{prop:theta}
For $f$ and $g$ as defined above, we have $R_2(f) < R_2(g)$.
Additionally, $R_3(g) < R_3(f)$ provided
\[
\frac{1}{|\cos(\theta/2)|^2} \leq 1 +  4\frac{\|\va_1\|_1\|\va_2\|_1}{\|\va\|_1^2}|\sin(\theta)|
\]
where $0 < \theta \leq \pi/2$ is the smallest angle between $\vw_1$, $\vw_2$, and $\va_1 = (\va_{1,k})_{k=1}^{K_1}$, $\va_2 = (\va_{1,k})_{k=1}^{K_2}$.
\end{proposition}
See \cref{app:proof52} for the proof.

\begin{figure}[ht!]
\centering
\includegraphics[width=.7\linewidth]{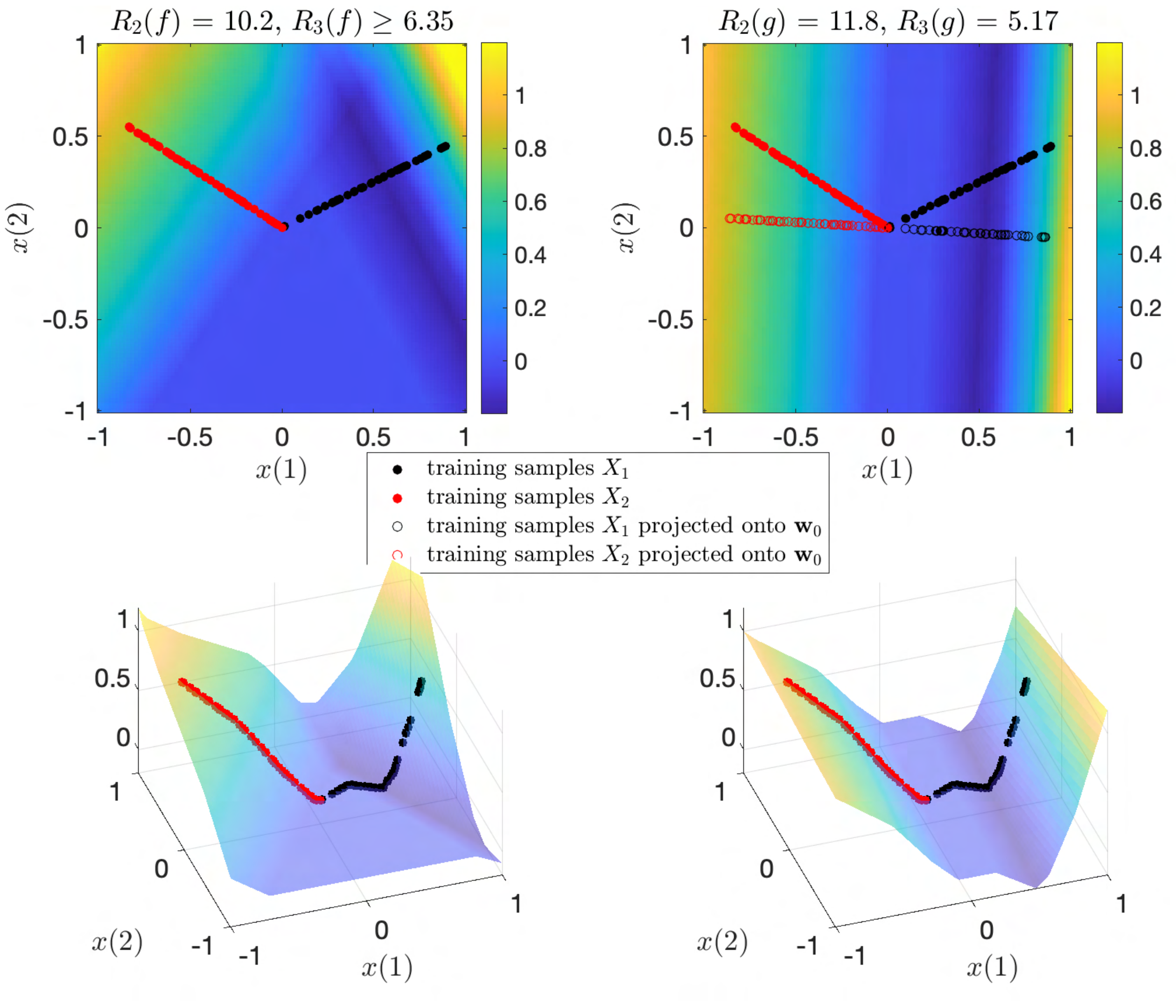}
\caption{\textbf{Effect of learning a minimum representation cost function interpolating training samples} with either a 2-layer ReLU network or a 3-layer network with a linear layer followed by ReLU layer. 
\textit{Left}: Minimum $R_2$ interpolant when samples lie on two rays separated by an angle of $2\pi/3$. 
\textit{Right}: Minimum $R_3$ interpolant of same data. 
Both functions are perfect interpolants; see \cref{fig:verify}.
When the training data may be interpolated by a function of the form $f(\mP_{\cal S}\vx)$, the linear layer promotes interpolating functions that do not vary in the direction perpendicular ${\cal S}$.}
\label{fig:multiunit}
\end{figure}

\begin{figure}[ht!]
\centering
\includegraphics[width=.7\linewidth]{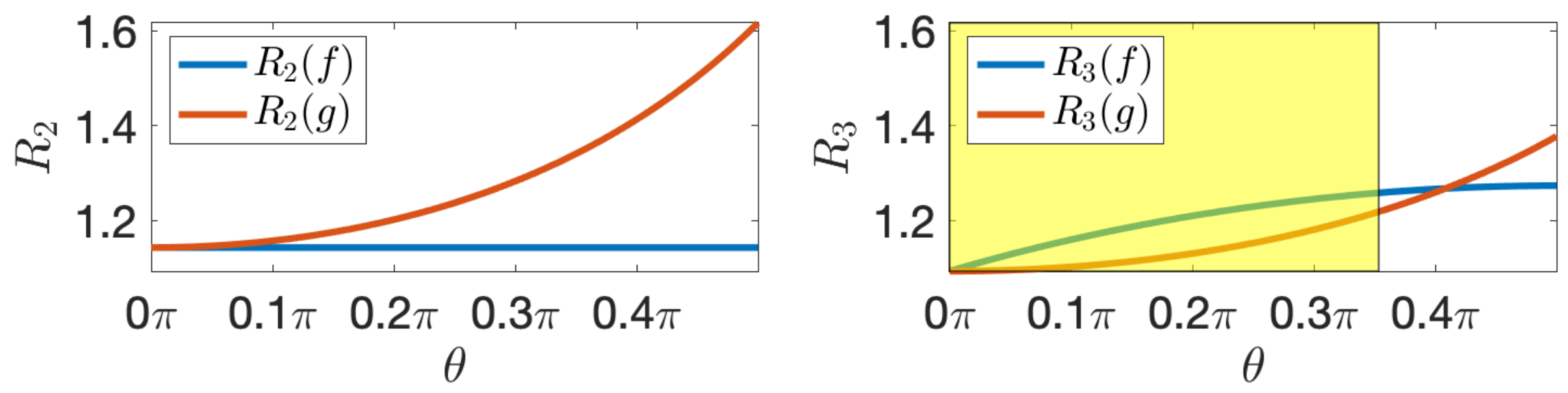}
\label{fig:theta}
\caption{\textbf{An interpolant with aligned units (induced by a three-layer network with a linear layer) may have a larger $R_3$ representation cost than an interpolant with unaligned units, depending on the training data distribution.} Data was generated as described in \cref{sec:rays} for various angles $\theta$ between the rays $\vw_1$ and $\vw_2$. For most values of $\theta$, the function $g$ with aligned units has the lower $R_3$ representation cost even when it has the higher $R_2$ representation cost, but for some values of $\theta$, the function $f$ with unaligned units has lower  $R_3$ \textit{and} lower $R_2$ representation costs. 
Examples of such functions are displayed in \cref{fig:counterintuitive}.
The yellow region highlights the range of $\theta$ for which \cref{prop:theta} predicts that the interpolant $g$ with aligned units will have lower $R_3$ representation cost than the interpolant $f$ with unaligned units.}
\end{figure}

This setting with samples on two rays is illustrated in \cref{fig:multiunit} with $\vw_{1} = [2,1]^\top$ and $\vw_{2}$
is $\vw_1$ rotated $120^{\circ}$ (before normalization). Note that 
both $f$ and $g$ fit the training samples exactly but have very different behavior away from the subspaces supporting the training data. 
In the two-layer setting, ReLU units may be unaligned with one another, and instead align with the support of the training samples (in this case $\vw_1$ and $\vw_2$) in order to minimize the $R_2$-cost, leading to complex behavior away from this support that presents challenges for out-of-distribution generalization analysis.

In contrast, the $R_3$-cost induced by adding a linear layer promotes  units that are aligned with a subspace ${\cal S}$,  such that there exists some $f \in N_2(\R^d)$ with $y_i = f(\mP_{\cal S}\vx_i) \; \forall i$. That is, even if the training samples do not lie on a subspace, 
if there is a subspace onto which the samples may be projected while still admitting an interpolant, then
that interpolant \textit{may} have a lower $R_3$-cost, depending on the angle $\theta$ between $\vw_1$ and $\vw_2$, as shown in \cref{fig:theta}. The resulting alignment of the ReLU units yields less complex behavior of $\hat f$ away from the support of the training samples, potentially leading to better out-of-distribution generalization. 

\subsection{General subspace projections}
\label{sec:projections}

Superficially, one might think the 
$R_3$ representation cost in \cref{lem:phiLinfdef}  associated with a network with linear layers composed with a two-layer ReLU network
is lower when $\mW$ has lower rank -- i.e. when the ReLU units are more aligned. However, the $R_3$ representation cost offers a more nuanced perspective because it highlights the interplay between the ReLU unit alignment with their scale (i.e.  $\va$). To see this, first note that for finite training samples, we may project them all into a subspace ${\cal S}$ for which no two distinct points are projected to the same location and then learn an interpolating function using these projected samples as feature vectors. As described in \cref{sec:interp}, the minimum $R_L$ interpolant found using this procedure will have all ReLU weight vectors $\vw_k \in {\cal S}$, so that $\rank(\mW) = \dim({\cal S})$. In other words, it is generally possible to find interpolating functions with aligned ReLU units. However, some ${\cal S}$ will lead to larger $R_3$ representation costs than others. As illustrated in \cref{fig:multi-project-med}, a poor choice of ${\cal S}$ (spanned by $\vw_0$ in the figure) will lead to a configuration of projected samples that can only be interpolated by a piecewise linear function with \textit{many} pieces (third column), while they may be interpolated with many fewer linear pieces for alternative ${\cal S}$ (first column with ${\cal S} = \R^2$ and second column with 
${\cal S} = \sspan(\vw_0)$). In other words, forcing ReLU units to lie in a poorly-chosen subspace may yield greater ReLU alignment by requiring many more units -- i.e., by sacrificing sparsity.

The representation cost in \eqref{eq:phi3} accounts for \textit{both} the alignment (through $\mW$ and the nuclear norm) and the sparsity (through $\va$). From here, we may infer that adding a linear layer to a two-layer ReLU network does more than ``promote alignment of ReLU units'':
\textbf{when searching for interpolants of finite training datasets, training a ReLU network with additional linear layers  implicitly seeks a \textit{low-dimensional} subspace such that a \textit{parsimonious} two-layer ReLU network can interpolate the projections of the training samples onto the subspace.}
Note that we are \textit{not} assuming that the training samples lie on a low-dimensional subspace; that is, we would not see the same effect by simply performing PCA on the training features before training the network. Rather, the best choice of subspace here depends heavily on the training labels.

\begin{figure}[ht!]
    \centering
    \includegraphics[width=.7\linewidth]{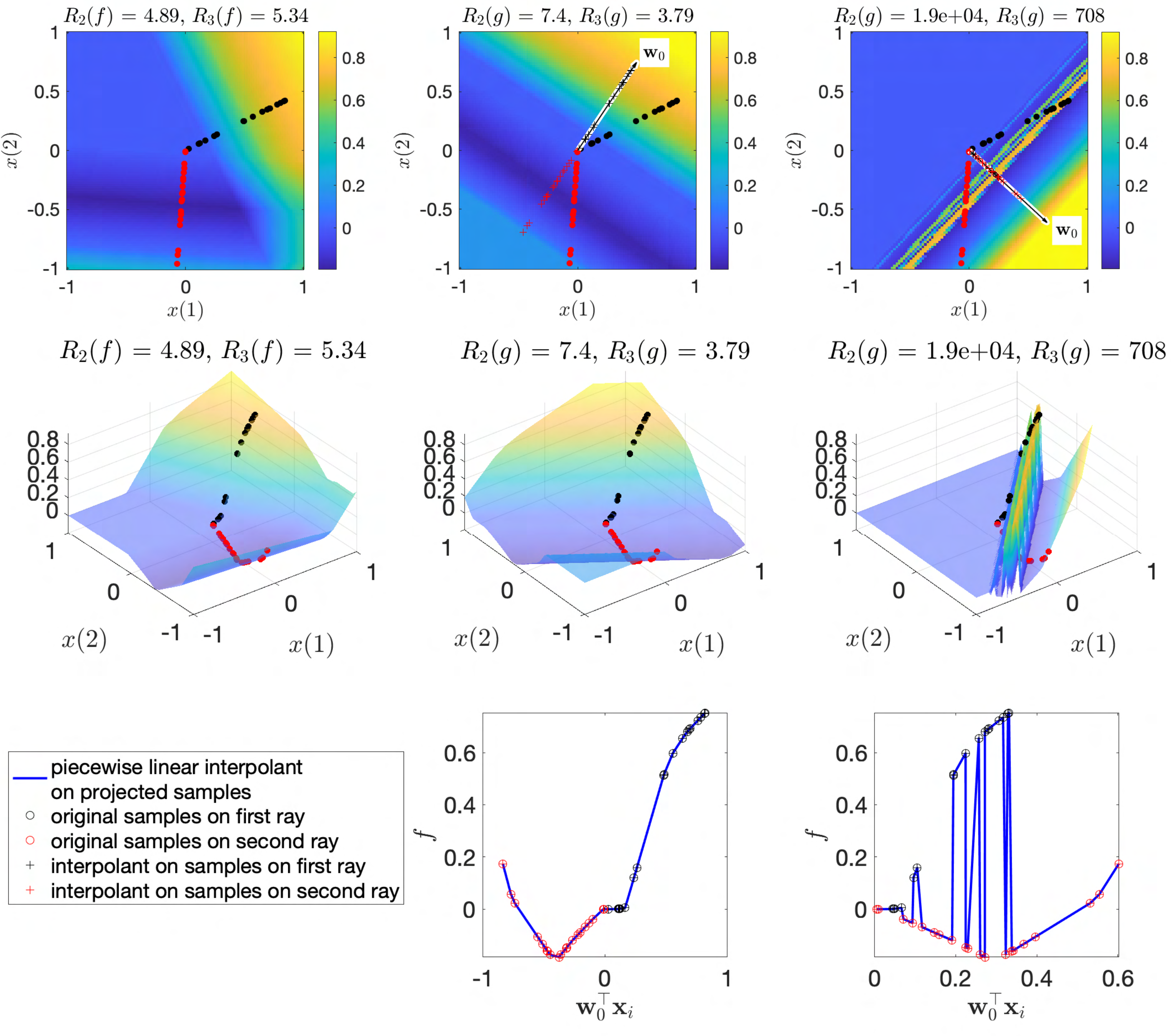}
    \caption{
    \textbf{Interpolating samples with aligned ReLU units} is akin to projecting samples onto a subspace (e.g., 1-D subspace $\vw_0$) and learning a piecewise linear interpolant on the subspace. Different choices of the subspace result in effective $\mW$s with the same nuclear norm but vastly different representation costs $R_3$. For instance, the second column shows the ideal subspace discussed in \cref{sec:rays}, yielding a small $R_3$-cost. In contrast, the third column shows a subspace choice that requires one to learn a piecewise linear function with many more pieces, yielding a much larger $R_3$-cost. These two examples have weight matrices $\mW$ with the same nuclear norm but different outer layer weights $\va$.  Complementary plots are in \cref{fig:multi-project-extra}.}
    \label{fig:multi-project-med}
\end{figure}

\section{Discussion}

Past work  exploring representation costs of neural networks either focused on two-layer networks \cite{savarese2019infinite,ongie2019function} or linear networks \cite{dai2021representation}. This paper is an important first step towards understanding the representation cost of \textit{nonlinear, multi-layer} networks. The representation cost expressions we derive offer new, quantitative insights into how multi-layer networks interpolate a finite set of training samples when trained using weight decay and reflect an interaction between ReLU unit alignment and sparsity that is not captured by past representation cost analyses.
Specifically, training a ReLU network with linear layers  implicitly seeks a \textit{low-dimensional} subspace such that a \textit{parsimonious} two-layer ReLU network can interpolate the projections of the training samples onto the subspace, {even when the samples themselves do not lie on a subspace.}
We note that ReLU alignment induced by linear layers leads to more predictable interpolant behavior off the training data support. Specifically, \cref{sec:subspaces} and \cref{sec:projections} show that
when the training data may be interpolated by a function of the form $f(\mP_{\cal S}\vx)$, using a linear layer promotes interpolating functions that do not vary in  directions orthogonal to ${\cal S}$.

\bibliography{refs}

\begin{thebibliography}{}

\bibitem[Arora et~al., 2018]{arora2018optimization}
Arora, S., Cohen, N., and Hazan, E. (2018).
\newblock On the optimization of deep networks: Implicit acceleration by
  overparameterization.
\newblock In {\em International Conference on Machine Learning}, pages
  244--253. PMLR.
\newblock \url{http://proceedings.mlr.press/v80/arora18a/arora18a.pdf}.

\bibitem[Arora et~al., 2019]{arora2019implicit}
Arora, S., Cohen, N., Hu, W., and Luo, Y. (2019).
\newblock Implicit regularization in deep matrix factorization.
\newblock {\em Advances in Neural Information Processing Systems},
  32:7413--7424.

\bibitem[Ba and Caruana, 2013]{ba2013deep}
Ba, L.~J. and Caruana, R. (2013).
\newblock Do deep nets really need to be deep?
\newblock {\em arXiv preprint arXiv:1312.6184}.

\bibitem[Dai et~al., 2021]{dai2021representation}
Dai, Z., Karzand, M., and Srebro, N. (2021).
\newblock Representation costs of linear neural networks: Analysis and design.
\newblock {\em Advances in Neural Information Processing Systems}, 34.

\bibitem[Daniely, 2017]{daniely2017depth}
Daniely, A. (2017).
\newblock Depth separation for neural networks.
\newblock In {\em Conference on Learning Theory}, pages 690--696. PMLR.

\bibitem[Golubeva et~al., 2020]{golubeva2020wider}
Golubeva, A., Neyshabur, B., and Gur-Ari, G. (2020).
\newblock Are wider nets better given the same number of parameters?
\newblock {\em arXiv preprint arXiv:2010.14495}.
\newblock \url{https://arxiv.org/pdf/2010.14495.pdf}.

\bibitem[Gunasekar et~al., 2018]{gunasekar2018implicit}
Gunasekar, S., Woodworth, B., Bhojanapalli, S., Neyshabur, B., and Srebro, N.
  (2018).
\newblock Implicit regularization in matrix factorization.
\newblock In {\em 2018 Information Theory and Applications Workshop (ITA)},
  pages 1--10. IEEE.

\bibitem[Hanin, 2019]{hanin2019universal}
Hanin, B. (2019).
\newblock Universal function approximation by deep neural nets with bounded
  width and relu activations.
\newblock {\em Mathematics}, 7(10):992.

\bibitem[Hanson and Pratt, 1988]{hanson1988comparing}
Hanson, S. and Pratt, L. (1988).
\newblock Comparing biases for minimal network construction with
  back-propagation.
\newblock {\em Advances in neural information processing systems}, 1:177--185.

\bibitem[Loshchilov and Hutter, 2017]{loshchilov2017decoupled}
Loshchilov, I. and Hutter, F. (2017).
\newblock Decoupled weight decay regularization.
\newblock {\em arXiv preprint arXiv:1711.05101}.

\bibitem[Mulayoff et~al., 2021]{mulayoff2021implicit}
Mulayoff, R., Michaeli, T., and Soudry, D. (2021).
\newblock The implicit bias of minima stability: A view from function space.
\newblock {\em Advances in Neural Information Processing Systems}, 34.

\bibitem[Neyshabur et~al., 2017]{neyshabur2017exploring}
Neyshabur, B., Bhojanapalli, S., Mcallester, D., and Srebro, N. (2017).
\newblock Exploring generalization in deep learning.
\newblock {\em Advances in Neural Information Processing Systems},
  30:5947--5956.

\bibitem[Neyshabur et~al., 2015]{neyshabur2015norm}
Neyshabur, B., Tomioka, R., and Srebro, N. (2015).
\newblock Norm-based capacity control in neural networks.
\newblock In {\em Conference on Learning Theory}, pages 1376--1401. PMLR.

\bibitem[Ongie et~al., 2019]{ongie2019function}
Ongie, G., Willett, R., Soudry, D., and Srebro, N. (2019).
\newblock A function space view of bounded norm infinite width relu nets: The
  multivariate case.
\newblock {\em arXiv preprint arXiv:1910.01635}.

\bibitem[Parhi and Nowak, 2021]{parhi2021banach}
Parhi, R. and Nowak, R.~D. (2021).
\newblock Banach space representer theorems for neural networks and ridge
  splines.
\newblock {\em J. Mach. Learn. Res.}, 22(43):1--40.

\bibitem[Razin and Cohen, 2020]{razin2020implicit}
Razin, N. and Cohen, N. (2020).
\newblock Implicit regularization in deep learning may not be explainable by
  norms.
\newblock {\em arXiv preprint arXiv:2005.06398}.

\bibitem[Razin et~al., 2021]{razin2021implicit}
Razin, N., Maman, A., and Cohen, N. (2021).
\newblock Implicit regularization in tensor factorization.
\newblock {\em arXiv preprint arXiv:2102.09972}.

\bibitem[Safran et~al., 2019]{safran2019depth}
Safran, I., Eldan, R., and Shamir, O. (2019).
\newblock Depth separations in neural networks: what is actually being
  separated?
\newblock In {\em Conference on Learning Theory}, pages 2664--2666. PMLR.
\newblock \url{http://proceedings.mlr.press/v99/safran19a/safran19a.pdf}.

\bibitem[Savarese et~al., 2019]{savarese2019infinite}
Savarese, P., Evron, I., Soudry, D., and Srebro, N. (2019).
\newblock How do infinite width bounded norm networks look in function space?
\newblock In {\em Conference on Learning Theory}, pages 2667--2690. PMLR.
\newblock \url{http://proceedings.mlr.press/v99/savarese19a/savarese19a.pdf}.

\bibitem[Shang et~al., 2020]{shang2020unified}
Shang, F., Liu, Y., Shang, F., Liu, H., Kong, L., and Jiao, L. (2020).
\newblock A unified scalable equivalent formulation for {S}chatten quasi-norms.
\newblock {\em Mathematics}, 8(8):1325.

\bibitem[Srebro et~al., 2004]{srebro2004maximum}
Srebro, N., Rennie, J.~D., and Jaakkola, T.~S. (2004).
\newblock Maximum-margin matrix factorization.
\newblock In {\em NIPS}, volume~17, pages 1329--1336. Citeseer.

\bibitem[Steinberg, 2005]{steinberg2005computation}
Steinberg, D. (2005).
\newblock Computation of matrix norms with applications to robust optimization.
\newblock {\em Research thesis, Technion-Israel University of Technology}, 2.

\bibitem[Urban et~al., 2016]{urban2016deep}
Urban, G., Geras, K.~J., Kahou, S.~E., Aslan, O., Wang, S., Caruana, R.,
  Mohamed, A., Philipose, M., and Richardson, M. (2016).
\newblock Do deep convolutional nets really need to be deep and convolutional?
\newblock {\em arXiv preprint arXiv:1603.05691}.

\bibitem[Vardi and Shamir, 2020]{vardi2020neural}
Vardi, G. and Shamir, O. (2020).
\newblock Neural networks with small weights and depth-separation barriers.
\newblock {\em arXiv preprint arXiv:2006.00625}.

\bibitem[Yuan and Lin, 2006]{yuan2006model}
Yuan, M. and Lin, Y. (2006).
\newblock Model selection and estimation in regression with grouped variables.
\newblock {\em Journal of the Royal Statistical Society: Series B (Statistical
  Methodology)}, 68(1):49--67.

\end{thebibliography}
\bibliographystyle{apalike}

\appendix
\onecolumn
\section{Proofs of Results in Section 3}\label{app:repcost}

\subsection{Proof of \cref{thm:univar}}
While \cref{thm:univar} can be proved by more direct means, we use the machinery developed in \cref{sec:simplify} to give a quick proof: In the univariate setting, the product of all linear-layers reduces to a column vector $\vw \in \R^{K\times 1}$, which is necessarily a rank-one matrix. Therefore, \cref{thm:univar} is a 
a direct consequence of \cref{cor:rankone}, which is a special case of \cref{prop:orthorows}, proved below.

\subsection{Proof of \cref{lem:schatten}}
The result is a direct consequence of the following variational characterization of the Schatten-$q$ quasi-norm for $q = 2/\ell$ where $\ell$ is a positive integer:
\[
\|\mW\|^{2/\ell}_{\Sc^{2/\ell}} = \min_{\mW = \mW_1\mW_2\cdots\mW_{\ell}} \frac{1}{\ell}\left(\|\mW_1\|_F^2 + \|\mW_2\|_F^2 + \cdots +  \|\mW_\ell\|_F^2\right)
\]
where the minimization is over all matrices $\mW_1,...,\mW_\ell$ of compatible dimensions. The case $\ell=2$ is well-known (see, e.g., \cite{srebro2004maximum}). The general case for $\ell \geq 3$ is established in \cite[Corollary 3]{shang2020unified}.


\subsection{Proof of \cref{lem:phiLinfdef}}
For any fixed $\bm\lambda > 0$, we may separately minimize over all scalar multiples $c\bm\lambda$ where $c>0$, to get 
\begin{align*}
\Phi_L(\mW,\va) & = \inf_{\bm\lambda > 0} 
\left( 
\inf_{c > 0}  
c^2 \tfrac{1}{L}\|\va\mD_{\bm\lambda}\|_2^2 
+ c^{-2/(L-1)}\tfrac{L-1}{L}\|\mD_{\bm\lambda}^{-1}\mW\|^{2/(L-1)}_{\Sc^{2/(L-1)}}
\right)\\
& = \inf_{\bm\lambda > 0} \left(\|\va \mD_{\bm\lambda}\|_2\|\mD_{\bm\lambda}^{-1}\mW\|_{\Sc^{2/(L-1)}}\right)^{2/L}
\end{align*}
The last step follows by the weighted AM-GM inequality: $ \tfrac{1}{L} a + \tfrac{L-1}{L} b \geq (ab^{L-1})^{1/L}$, which holds with equality when $a = b$. Here we have $a = (c \|\va\mD_{\bm\lambda}\|)^2$ and $b = \left(c^{-1}\|\mD_{\bm\lambda}^{-1}\mW\|_{\Sc^{2/(L-1)}}\right)^{2/(L-1)}$, and there exists a $c>0$ for which $a = b$, hence we obtain the lower bound.

Finally, performing the invertible change of variables $\lambda_i = \lambda_i'/a_i$, we have $\mD_{\bm\lambda'}^{-1} = \mD_{\bm\lambda}^{-1}\mD_\va $, and so
\begin{align}
\Phi_L(\mW,\va) & = \inf_{\bm\lambda' > 0} \left(\|\bm\lambda'\|_2\|\mD_{\bm\lambda'}^{-1}\mD_\va\mW\|_{\Sc^{2/(L-1)}}\right)^{2/L}\\
& = \inf_{\substack{\bm\lambda' > 0\\\|\bm\lambda'\|_2 = 1}} \|\mD_{\bm\lambda'}^{-1}\mD_\va\mW\|_{\Sc^{2/(L-1)}}^{2/L}
\end{align}
where we are able to constrain $\bm\lambda'$ to be unit norm since $\|\bm\lambda'\|_2\|\mD_{\bm\lambda'}^{-1}\mD_\va\mW\|_{\Sc^{2/(L-1)}}$ is invariant to scaling  $\bm\lambda'$ by positive constants.

\subsection{Proof of \cref{prop:orthorows}}
We use the variational characterization of $\Phi_L(\mW,\va)$ given in \cref{lem:phiLinfdef} as an infimum over the Schatten-$q$ quasi-norm of matrices of the form $\mD_{\vlambda}^{-1}\mD_\va \mW$. Fix a vector $\bm\lambda \in \R^K$ with positive entries. We begin by constructing an SVD of the matrix $\mD_{\vlambda}^{-1}\mD_\va \mW$. Let $\mV = [\vv_1 \cdots \vv_m]\in\R^{d\times m}$.
Observe that $\mW$ factors as $\mW = \mU\mV^\T$ where $\mU = [\vu_1 \cdots \vu_m]\in \R^{K\times m}$ is such that $[\vu_j]_k = \pm 1$ when the $k$th row of $\mW$ is equal to $\pm \vv_j$ and $[\vu_j]_k =0$ otherwise. In particular, the vectors $\vu_1,...,\vu_m$ are mutually orthogonal, and likewise so are the vectors $\mD_{\vlambda}^{-1}\mD_\va \vu_1,...,\mD_{\vlambda}^{-1}\mD_\va \vu_m$. 
We also have 
$\|\mD_{\vlambda}^{-1}\mD_\va \vu_j\|_2 = \|\mD_{{\vlambda}_j}^{-1}\va_j\|_2 :=\sigma_j$ where ${\vlambda}_j$ for all $j=1,...,m$ denotes the restriction of $\vlambda$ to the subset of entries corresponding to rows of $\mW$ equal to $\pm \vv_j$. Let $\bm\Sigma = \mD_{\bm\sigma}$ where $\bm\sigma = (\sigma_1,...,\sigma_m)$. Then $\hat{\mU} = \mU {\bm\Sigma}^{-1}$ has orthonormal columns, and so does $\mV$ by assumption. This shows an SVD of $\mD_{\vlambda}^{-1}\mD_\va \mW$ is given by
\[
\mD_{\vlambda}^{-1}\mD_\va \mW = \hat{\mU}\bm\Sigma \mV^\T
\]
In particular, $\bm\sigma$ are the singular values of $\mD_{\vlambda}^{-1}\mD_\va \mW$. Therefore, for any $q>0$ we have
\[
\|\mD_{\vlambda}^{-1}\mD_\va \mW\|_{\Sc^{q}}^q = \sum_{j=1}^m \|\mD_{{\vlambda}_i}^{-1}\va_j\|_2^q.
\]
Let $q = 2/(L-1)$. Then we have
\[
\Phi_L(\mW,\va)^{\frac{L}{L-1}} = 
\inf_{\substack{\bm \lambda > 0\\ \|\bm\lambda\|_2 = 1} } \|\mD_{\vlambda}^{-1}\mD_\va \mW\|_{\Sc^{q}}^q = \inf_{\substack{\bm \lambda > 0\\ \|\bm\lambda\|_2 = 1} } \sum_{j=1}^m \|\mD_{{\vlambda}_j}^{-1}\va_j\|_2^q.
\]
The inf on the right-hand side is equivalent to
\begin{align*}
\inf_{\substack{\bm c_j > 0 \\ \sum_{j=1}^m c_j^2 = 1}} \inf_{\substack{\bm \lambda_j > 0\\ \|\bm\lambda_j\|_2 = c_j}} \sum_{j=1}^m 
 \|\mD_{{\vlambda}_j}^{-1}\va_j\|_2^q & = \inf_{\substack{\bm c_j > 0 \\ \sum_{j=1}^m c_j^2 = 1}} \sum_{j=1}^m \inf_{\substack{\bm \lambda_j > 0\\ \|\bm\lambda_j\|_2 = c_j}} 
 \|\mD_{{\vlambda}_j}^{-1}\va_j\|_2^q\\
& =  \inf_{\substack{\bm c_j > 0 \\ \sum_{j=1}^m c_j^2 = 1}} \sum_{j=1}^m 
 \left(\frac{1}{c_j}\|\va_j\|_1\right)^q\\
 & = \left(\sum_{j=1}^m \|\va_j\|_1^{2/L}
 \right)^{\frac{L}{L-1}}
\end{align*}
where the final equality follows from an application of Lemma 3.1 in \cite{steinberg2005computation}, proving the claim.

\subsection{Proof of \cref{lem:Q}}
Starting from \cref{lem:phiLinfdef}, we have 
\[
\Phi_3(\mW.\va) = \inflam \|\mD_{\vlambda}^{-1}\mD_\va \mW\|_*^{2/3}.
\]
Since the operator norm is the convex dual of the nuclear norm, the right-hand-side above is equivalent to
\[
\inflam \max_{\|\mQ\|\leq 1} \langle \mD_{\vlambda}^{-1}\mD_\va \mW, \mQ\rangle^{2/3}
\]
where the maximum is over all $\mQ$ having the same dimensions as $\mQ$ and $\|\mQ\|$ denotes the spectral norm of $\mQ$, and $\langle\mW,\mQ\rangle = \Tr \mW^\T\mQ$. Equivalently, expanding the trace inner product in terms of the rows of $\mW$ and $\mQ$ we have
\[
\Phi_3(\mW,\va) = \inflam  \max_{\|\mQ\| \leq  1} \left( \sum_{k=1}^K \frac{a_k \langle \vw_k,\vq_k\rangle}{\lambda_k}  \right)^{2/3}
\]
where $\vw_i$ and $\vq_i$ denote the $i$th row of $\mW$ and $\mQ$, respectively. By Sion's minimax theorem, we may exchange the order of the inf and max, to get
\[
\Phi_3(\mW,\va) =  \max_{\|\mQ\| \leq  1} \inflam \left( \sum_{k=1}^K \frac{a_k \langle \vw_k,\vq_k\rangle}{\lambda_k}  \right)^{2/3} = \max_{\|\mQ\| \leq  1} \sum_{k=1}^K |a_k \langle \vw_k,\vq_k\rangle|^{2/3}
\]
where the final equality follows from an application of Lemma 3.1 in \cite{steinberg2005computation}.

\section{Proofs of Results in Section 4}\label{app:mincost}
\subsection{Proof of  \cref{prop:colineardata}}
This result is most easily shown using the initial formulation of the data interpolating problem \eqref{eq:opt1} as minimizing the Euclidean norm of the weights $C_L(\theta)$ in an $L$-layer representation.

Let $f = h^{(L)}_\theta$ be any minimum $C_L(\theta)$ interpolant of the training data, whose first layer weight matrix is $\mW_1$. Let $\Sc$ be the subspace spanned by the training data locations, and let $\mP_{\Sc}$ be the orthogonal projector onto $\Sc$. We will show that $\mW_1 = \mW_1 \mP_{\Sc}$, and so $f(\vx) = f(\mP_{\Sc}\vx)$.

Suppose, by way of contradiction, that $\mW_1 \neq \mW_1 \mP_{\Sc}$, or equivalently, $\mW_1(\mI-\mP_{\Sc})\neq \bm 0$. By the Pythagorean theorem we have
\[
\|\mW_1\|_F^2 = \|\mW_1\mP_{\Sc}\|_F^2 + \|\mW_1(\mI-\mP_{\Sc})\|_F^2
\]
which implies 
\[
\|\mW_1\|_F^2  >  \|\mW_1\mP_{\Sc}\|_F^2
\]
where the inequality is strict since by the assumption  $\mW_1(\mI-\mP_{\Sc}) \neq \bm 0$. Also, since $\vx_i\in\Sc$ for all $i=1,...,n$, we have $\mW_1\vx_i = \mW_i\mP_{\Sc}\vx_i$. Therefore, the cost $C_L(\theta)$ is always strictly reduced by replacing $\mW_1$ with $\widetilde{\mW}_1 = \mW_1\mP_{\Sc}$, while the data fit is left unchanged, violating the assumption that $\mW_1$ belonged to a minimizing set of parameters, which proves the claim.

\subsection{Proof of  \cref{prop:colineardata2}}
Without loss of generality, we may translate the data points so that they lie on a one-dimensional linear subspace.

Let $f(\vx) = \va^\T[\mW\vx+\vb]_+ + c$ be any minimum $R_L$-norm interpolant. Because we can rescale the rows of $\mW$ and the corresponding entries of $\va$ without changing the representation cost, we may assume $\mW$ has unit-norm rows. Also, since the data is co-linear,  \cref{prop:colineardata} shows that
\[
f(\vx) = f(\mP_{\Sc} \vx)
\]
where $\Sc$ is the one-dimensional subspace spanned by the data locations. This implies that the inner-layer weight matrix $\mW$ is rank one of the form $\mW = \vs\vu^\T$. Therefore, by  \cref{cor:rankone}, we have
\[
\Phi_L(\mW,\va) = \|\va\|_1^{2/L}
\]
Since $\|\va\|_1^{2/L}$ is a monotonic transformation of $\|\va\|_1$, the sets of minimizing $R_L$-interpolating solutions for all $L\geq 2$ must coincide, similar to the univariate case.

\section{Proofs of Results in Section 5}\label{app:example}

\subsection{Proof of  \cref{prop:rankoneR3}}\label{app:proof51}
We first prove two lemmas:
\begin{lemma}\label{lem:minlam}
For any vector $\va \in \R^K$ we have
\[
\inf_{\substack{\bm \lambda > 0\\ \|\bm\lambda\|_2 = 1} } \|\mD^{-1}_{\bm\lambda} \va\|_2^2 = \|\va\|_1^2
\]
\end{lemma}
\begin{proof}
This can be shown directly by optimizing the Lagrangian corresponding to the squared objective:
\[
\mathcal{L}(\bm\lambda,\sigma) = \sum_{k} \frac{|a_k|^2}{\lambda_k^2}+ \sigma\left(1-\sum_k \lambda_k^2\right)
\]
Straightforward calculations show that the only critical point of the Lagrangian is where $\lambda_k = |a_k|^{1/2}/\|\va_1\|^{1/2}$, which gives the claim.
\end{proof}
\begin{lemma}\label{lem:minlam2} 
\[
\inf_{\substack{\bm\lambda \in \R^K\\ \bm \lambda > 0\\ \|\bm\lambda\|_2 = 1} } \prod_{k=1}^K \frac{1}{\lambda_k^2} = K^{K}.
\]
\end{lemma}
\begin{proof}
We show the inf is attained where $\lambda_1 = \cdots = \lambda_K = 1/\sqrt{K}$. Suppose it is not, so that there exists indices $i\neq j$ such that $\lambda_i \neq \lambda_j$. Holding all the other $\lambda_k$ fixed, consider all $\alpha,\beta>0$ subject to the constraints $\alpha^2 + \beta^2 = C$ where $C=1-\sum_{k\neq i,j}\lambda_k^2$ is constant. Then the optimizer of
\[
\inf_{\substack{\alpha,\beta > 0\\ \alpha^2 + \beta^2 = C}} \frac{1}{\alpha\beta}
\]
occurs where $\alpha = \beta$, hence by replacing $\lambda_i$ with $\lambda_j$ with their common value, we would be able to reduce the original objective, violating the assumption that $\lambda_i$ and $\lambda_j$ were minimizers. Therefore, $\lambda_i = \lambda_j$ for all $i,j$, which implies $\lambda_k = 1/\sqrt{K}$ for all $k$.
\end{proof}

Now we prove \cref{prop:rankoneR3}. Let $\va \in \R^K$ and suppose $\mW \in \R^{K\times d}$ is such that $r:= \rank(\mW) > 1$. Fix a unit-norm weighting vector $\bm\lambda$ with positive entries. Let $\sigma_1,...,\sigma_r$ be the singular values of the matrix $\mD_{\bm\lambda}^{-1}\mD_\va \mW$. Then we have
\[
\|\mD^{-1}_{\bm\lambda} \mD_\va \mW\|_*^2 = \left(\sum_{i=1}^r \sigma_i \right)^2 = \sum_{i=1}^r \sigma_i^2 + 2\sum_{i > j}\sigma_i\sigma_j
\]
Note that 
\[
\left(\sum_{i=1}^r \sigma_i \right)^2 = \|\mD^{-1}_{\bm\lambda} \mD_\va \mW\|_F^2 = \|\mD^{-1}_{\bm\lambda} \va\|_2^2
\]
and by  \cref{lem:minlam} we have
\[
\inf_{\substack{\bm \lambda > 0\\ \|\bm\lambda\|_2 = 1}} \sum_{i>j} \|\mD^{-1}_{\bm\lambda} \va\|_2^2 = \|\va\|_1^2
\]
Also, since there are $r(r-1)/2$ terms in the sum $\sum_{i >  j}\sigma_i\sigma_j$, and each $\sigma_i$ appears as a factor $r-1$ times , by the AM-GM inequality we have
\[
\sum_{i >  j}\sigma_i\sigma_j \geq \frac{2}{r(r-1)}\left(\prod_{i=1}^r \sigma_i^2\right)^{\frac{1}{r}} 
\]
Since $\mQ:= \mD_{\bm\lambda}\mD_\va\mW$ is rank $r$ there exists a subset of $r$ rows of $\mQ$ such that its restriction to these rows is also rank $r$. Collect the indices of these rows into a set $\Omega$, and let $\mP_\Omega \in \R^{r\times K}$ be the matrix that restricts  onto indices in $\Omega$. If we let $\tilde{\sigma}_i$ denote the $i$th singular value of $\mP_{\Omega}\mQ$, observe that $0 < \tilde{\sigma}_i \leq \sigma_i$ for all $i=1,...,r$. Therefore,
\[
\prod_{i=1}^r \sigma_i^2 \geq \prod_{i=1}^r \tilde{\sigma}_i^2 = \det(\mP_{\Omega}\mQ\mQ\mP_{\Omega}^\T) = \frac{\prod_{k\in\Omega}|a_k|^2 }{\prod_{k\in\Omega}\lambda_k^2}\det(\mP_\Omega\mW\mW^\T\mP_\Omega^\T)
\]
Also, by  \cref{lem:minlam2}
\[
\inf_{\substack{\bm \lambda > 0\\ \|\bm\lambda\|_2 = 1}} \frac{1}{\prod_{k\in\Omega}\lambda_k^2} = r^r
\]
since we can shrink all $\lambda_k$ with $k\not\in \Omega$ to zero. Therefore, setting $C = (\prod_{k\in\Omega} |a_k|^2) \det(P_\Omega \mW\mW^TP_\Omega^T) > 0$, we have shown
\[
\inf_{\substack{\bm \lambda > 0\\ \|\bm\lambda\|_2 = 1}} 2\sum_{i>j} \sigma_i\sigma_j \geq \frac{4C^{1/r}}{r-1}
\]

Finally, putting the above pieces together, we have
\begin{align}
\Phi_3(\mW,\va)^{3} = \inf_{\substack{\bm \lambda > 0\\ \|\bm\lambda\|_2 = 1}}
\|\mD^{-1}_{\bm\lambda} \mD_\va \mW\|_*^2 
& = 
\inf_{\substack{\bm \lambda > 0\\ \|\bm\lambda\|_2 = 1}} \sum_{i=1}^r \sigma_i^2 + 2\sum_{i > j}\sigma_i\sigma_j\\
& \geq \inf_{\substack{\bm \lambda > 0\\ \|\bm\lambda\|_2 = 1}} \sum_{i=1}^r \sigma_i^2 + \inf_{\substack{\bm \lambda > 0\\ \|\bm\lambda\|_2 = 1}} 2\sum_{i > j}\sigma_i\sigma_j\\
& \geq \|\va\|_1^2 + \frac{4C^{1/r}}{r-1}
\end{align}
and since $\frac{4C^{1/r}}{r-1} > 0$ we see that $\Phi_3(\mW,\va) > \|\va\|_1^{2/3}$, as claimed.

\subsection{Proof of  \cref{prop:theta}}\label{app:proof52}

First, since $g$ has rank-one weights, we can compute $\Phi_3(g)$ exactly using  \cref{cor:rankone}:
\[
\Phi_3(g)= \left(|\langle \vw_0, \vw_1\rangle|^{-1}\sum_{k=1}^{K_1} \|\va_1\|_1 + |\langle \vw_0, \vw_2\rangle|^{-1}\|\va_2\|_1\right)^{2/3} = \left(\frac{\|\va\|_1} {\cos(\theta/2)}\right)^{2/3}
\]

Next we lower bound $\Phi_3$ using the equivalence
\[
\Phi_3(f) = \inf_{\substack{\bm \lambda > 0\\ \|\bm\lambda\| = 1} } \|\mD^{-1}_{\bm\lambda} \mD_\va \mW\|_*^{2/3}.
\]
Fix any $\bm\lambda > 0$. We begin by computing an SVD of $\mM = \mD_{\bm\lambda}^{-1}\mD_{\va}\mW$. Observe that $\mW = \mU\begin{bmatrix}\vw_1^\T\\
\vw_2^\T
\end{bmatrix}$ where $\mU = [\vu_1 \vu_2]$ with $(\vu_1)_k = 1$ when the $k$th row of $\mW$ is $\vw_1$ and $0$ otherwise, and vice-versa for $\vu_2$. In particular, $\vu_1$ and $\vu_2$ are orthogonal, and so are $\mD_{\lambda}^{-1}\mD_\va\vu_1$ and $\mD_{\lambda}^{-1}\mD_{\va}\vu_2$. Note that the norms of these vectors are $\|\mD_{\bm\lambda_1}^{-1}\va_1\|$ and $\|\mD_{\bm\lambda_2}^{-1}\va_2\|$, where $\mD_{\bm\lambda_1}$ and $\mD_{\bm\lambda_2}$ denote restrictions of $\bm\lambda$ to indices corresponding to the weights in $\va_1$ and $\va_2$, respectively.

Let $\hat{\mU}\hat{\bm\Sigma}\hat{\mV}^\T$ be an SVD of the  $2 \times 2$ matrix \[\mQ :=
\begin{bmatrix}
\|\mD_{\bm\lambda_1}^{-1}\va_1\|_2\vw_1^\T\\
\|\mD_{\bm\lambda_2}^{-1}\va_2\|_2\vw_2^T
\end{bmatrix}
= 
\begin{bmatrix}
\|\mD_{\bm\lambda_1}^{-1}\va_1\|_2 & 0\\
0 & \|\mD_{\bm\lambda_2}^{-1}\va_2\|_2
\end{bmatrix}
\begin{bmatrix}
\vw_1^\T\\
\vw_2^T
\end{bmatrix}
\]
Then an SVD of $\mD_\va \mW$ is
\[
\mD_\va \mW = \left( \mD_{\va}\mU
\begin{bmatrix}
\|\mD_{\bm\lambda_1}^{-1}\va_1\|_2^{-1} & 0\\
0 & \|\mD_{\bm\lambda_2}^{-1}\va_2\|_2^{-1}
\end{bmatrix}
\hat{\mU}\right) \hat{\bm\Sigma}\hat{\mV}^\T
\]

This shows that the singular values of $\mM$ coincide with those of the matrix $\mQ$. Let $\sigma_1$ and $\sigma_2$ be their two non-zero singular values.
Then, we have the identities:
\begin{align}
(\sigma_1+\sigma_2)^2 & =  \sigma_1^2+\sigma_2^2 + 2\sigma_1\sigma_2\\
& = \|\mQ\|_F^2 + 2\det(\mQ\mQ^\T)^{1/2}\\
& = \|\mD_{\bm\lambda}^{-1}\va\|_2^2 + 2\|\mD_{\bm\lambda_1}^{-1} \va_1\|_2\|\mD_{{\bm\lambda}_2}^{-1} \va_2\|_2|\sin(\theta)|
\end{align}
Next we optimize the last two terms \emph{separately} over $\bm\lambda$ to get a lower bound. These can be optimized exactly using simple Lagrange multiplier arguments, to give:
\[
\inf_{\substack{\bm\lambda>0\\ \|\bm\lambda\|_2 = 1}}\|\mD^{-1}_{\bm\lambda} \va\|_2^2 = \|\va\|_1^2
\]
and
\[
\inf_{\substack{\bm\lambda_1,\bm\lambda_2>0\\ \|\bm\lambda_1\|_2^2 +  \|\bm\lambda_2\|_2^2 = 1}}\|\mD_{\bm\lambda_1}^{-1} \va_1\|_2\|\mD_{{\bm\lambda}_2}^{-1} \va_2\|_2 = 2\|\va_1\|_1\|\va_2\|_1
\]
Therefore, we have shown
\begin{align}
\Phi_3(f)& \geq \left(\|\va\|_1^2 +  4\|\va_1\|_1\|\va_2\|_1|\sin(\theta)|\right)^{1/3}\\
& = \|\va\|_1^{2/3}\left(1+  4\frac{\|\va_1\|_1\|\va_2\|_1}{\|\va\|_1^2}|\sin(\theta)|\right)^{1/3}
\end{align}
Finally, comparing this to the expression for $\Phi_3(g)$ and cancelling the factor of $\|\va\|_1^{2/3}$ gives the claim.

\section{Supplemental figures}

\begin{figure}[ht!]
\centering
\includegraphics[width=.5\linewidth]{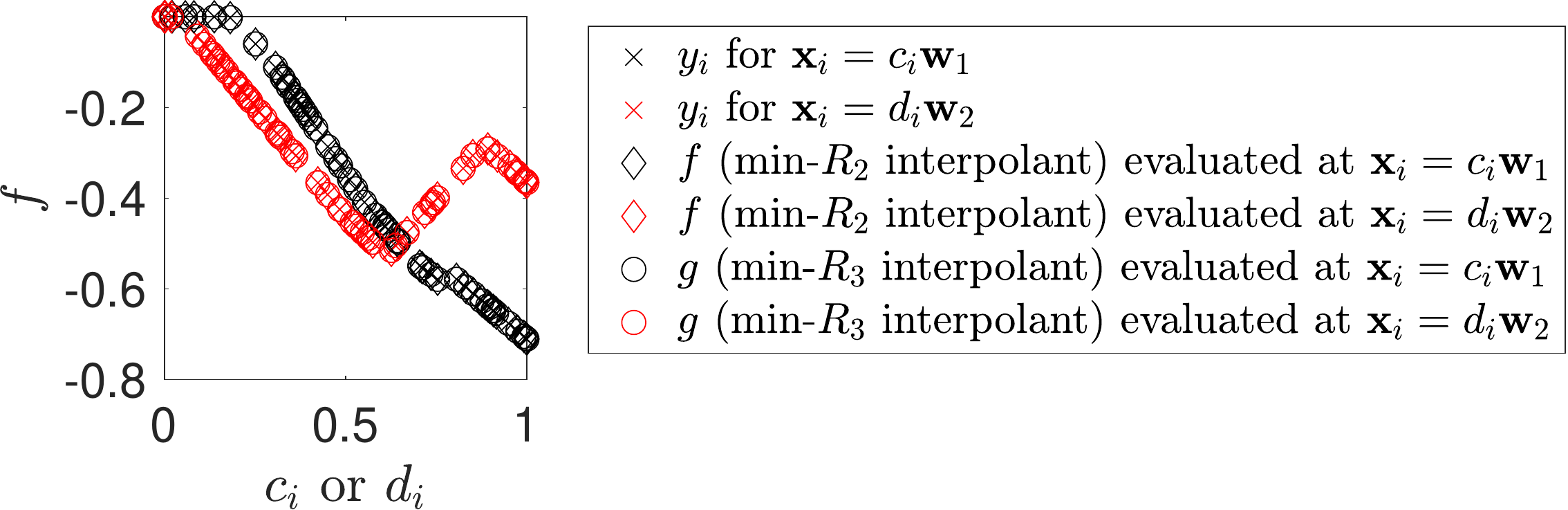}
\caption{Empirical confirmation that both functions $f$ and $g$ in \cref{fig:multiunit} are perfectly interpolating the training data.}
\label{fig:verify}
\end{figure}

\begin{figure*}[ht!]
    \centering
    \includegraphics[width=.9\linewidth]{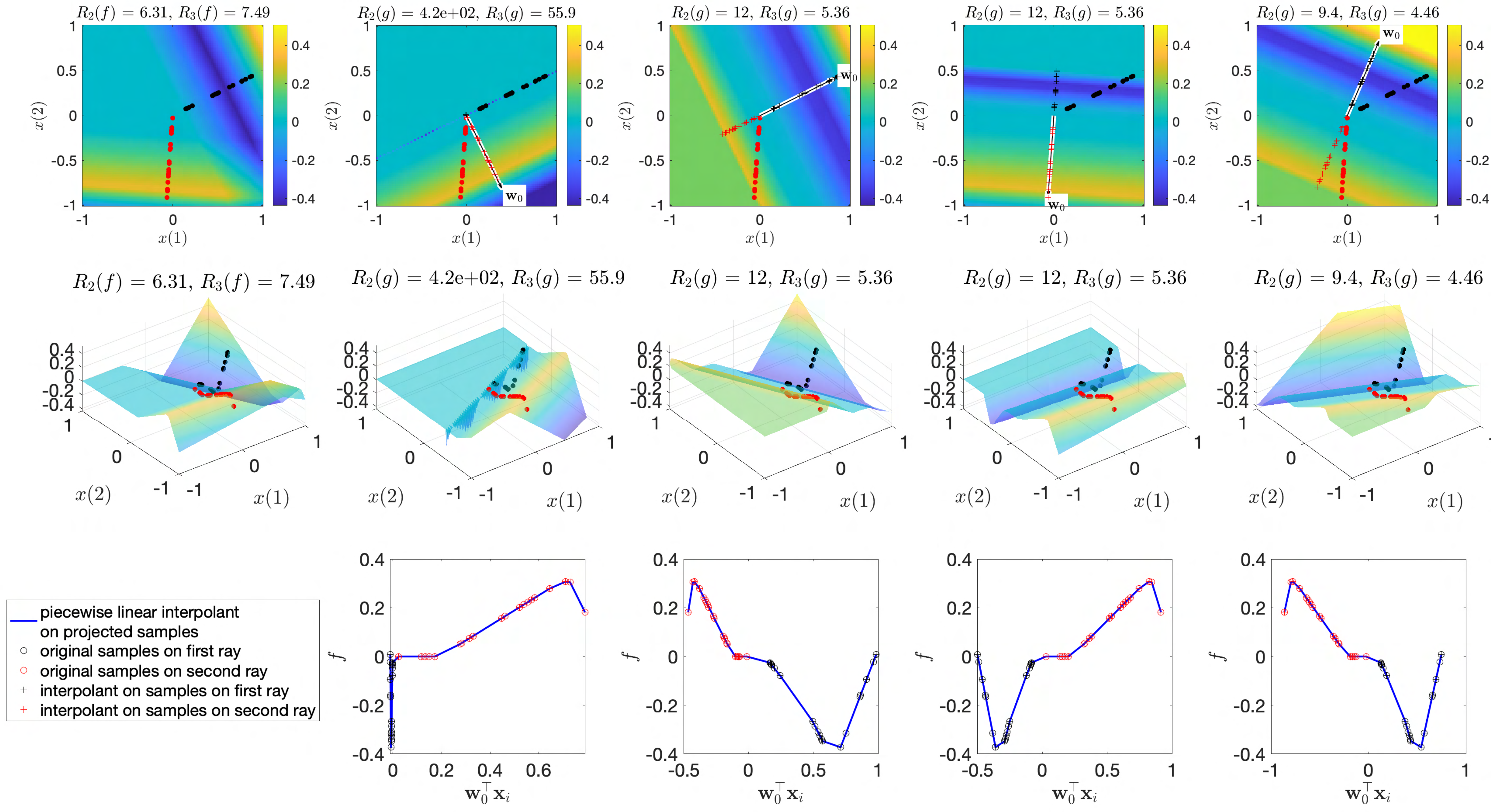} 
    \caption{\textbf{Interpolating samples with aligned ReLU units} like \cref{fig:multi-project-med} with different subspaces $\vw_0$.}
    \label{fig:multi-project-extra}
\end{figure*}

\end{document}